\documentclass[final,3p,12pt]{elsarticle}

\usepackage{amssymb}
\usepackage{amsmath}

\usepackage{etoc}
\usepackage{booktabs}
\usepackage{threeparttable}
\usepackage{graphicx}
\usepackage{subcaption}

\usepackage{amsthm}
\newtheorem{theorem}{Theorem}
\newtheorem{corollary}{Corollary}

\theoremstyle{remark}
\newtheorem{remark}{Remark}

\theoremstyle{definition}
\newtheorem{definition}{Definition}
\newtheorem{lemma}[theorem]{Lemma}

\usepackage{xcolor}
\usepackage{amsmath,amssymb,bm}   
\usepackage[ruled,vlined,linesnumbered]{algorithm2e}
\usepackage{url}

\journal{Pre-print (under review)}

\newcommand{\boldX}{\boldsymbol{X}}

\DeclareMathOperator*{\argmin}{\arg\!\min}

\newcommand{\widesim}[2][1.5]{\mathrel{\overset{#2}{\scalebox{#1}[1]{$\sim$}}}}
\newcommand{\iid}{\widesim[3]{i.i.d.}}

\begin{document}

\begin{frontmatter}



\title{Federated Ensemble Learning with Progressive
\\
Model Personalization} 

\author{Ala Emrani}
\ead{ala.emrani94@sharif.edu}

\author{Amir Najafi}
\ead{amir.najafi@sharif.edu}

\author{Abolfazl Motahari}
\ead{motahari@sharif.edu}


\affiliation{organization={Department of Computer Engineering,\\
Sharif University of Technology},
city={Tehran},
country={Iran}}

\begin{abstract}
Federated Learning provides a privacy-preserving paradigm for distributed learning, but suffers from statistical heterogeneity across clients. Personalized Federated Learning (PFL) mitigates this issue by considering client-specific models. A widely adopted approach in PFL decomposes neural networks into a shared feature extractor and client-specific heads. While effective, this design induces a fundamental tradeoff: deep or expressive shared components hinder personalization, whereas large local heads exacerbate overfitting under limited per-client data. Most existing methods rely on rigid, shallow heads, and therefore fail to navigate this tradeoff in a principled manner. In this work, we propose a boosting-inspired framework that enables a smooth control of this tradeoff. Instead of training a single personalized model, we construct an ensemble of $T$ models for each client. Across boosting iterations, the depth of the personalized component are progressively increased, while its effective complexity is systematically controlled via low-rank factorization or width shrinkage. This design simultaneously limits overfitting and substantially reduces per-client bias by allowing increasingly expressive personalization. We provide theoretical analysis that establishes generalization bounds with favorable dependence on the average local sample size and the total number of clients. Specifically, we prove that the complexity of the shared layers is effectively suppressed, while the dependence on the boosting horizon $T$ is controlled through parameter reduction. Notably, we provide a novel nonlinear generalization guarantee for decoupled PFL models. Extensive experiments on benchmark and real-world datasets (e.g., EMNIST, CIFAR-10/100, and Sent140) demonstrate that the proposed framework consistently outperforms state-of-the-art PFL methods under heterogeneous data distributions.\footnote{The codes are available at \url{https://anonymous.4open.science/r/PPFE-6E7C}.}
\end{abstract}



\begin{keyword}
Personalized Federated Learning, Boosting, Learning Theory
\end{keyword}

\end{frontmatter}


\section{Introduction}
\label{sec:intro}

Federated learning is a distributed paradigm that enables a collection of clients to collaboratively train machine learning models using locally stored data, without direct data sharing, thereby preserving privacy. In its original formulation, McMahan et al.\ \cite{mcmahan2017communication} proposed learning a single global model shared among all clients. Although effective under the assumption of independent and identically distributed (IID) data, the performance of this approach deteriorates substantially in the presence of distributional heterogeneity across clients. In real-world deployments, clients typically collect data under diverse conditions, from heterogeneous populations, and with varying feature--label relationships. Such non-IID characteristics render a single global model suboptimal for many clients. To alleviate the adverse effects of data heterogeneity during training and inference, a substantial body of work has proposed algorithmic and optimization-based remedies \cite{karimireddy2020scaffold, li2020fedprox, wang2020tackling, li2021model, wu2024fedel}. Nevertheless, the inherent limited expressiveness of a single global model in highly heterogeneous environments has motivated the development of the personalized federated learning (PFL) framework. PFL allows clients to jointly exploit shared global knowledge across distributed datasets while simultaneously learning customized models that are better aligned with local data distributions \cite{arivazhagan2019federated, t2020personalized, fallah2020personalized}.

The goal of PFL is to strike a balance between global knowledge sharing and local adaptation. A common architectural strategy, motivated by centralized multi-task learning, involves decoupling the neural network into two components: a \emph{shared~body} that serves as a representation extractor, and a \emph{personalized~head} tailored to each client’s task \cite{Oh2021FedBABUTE, zheng2025personalized}. The body transforms input data into an embedding space, under the assumption that a lightweight head model can effectively complete the task. Within this framework, clients collaboratively learn a shared body while independently optimizing their task-specific heads. From a systems perspective, this decoupling enables scalable personalization while maintaining a shared representation across heterogeneous clients. This approach, exemplified by methods such as FedRep \cite{collins2021exploiting}, has demonstrated considerable promise. 

However, existing theoretical guarantees for such methods are largely confined to simplified linear models, leaving their behavior in practical nonlinear and deep learning regimes poorly understood. Moreover, much of the prevailing intuition is derived from centralized experiments and does not fully capture the operational constraints of heterogeneous federated systems. In federated environments, additional flexibility is required to allow clients to learn representations that better align with their local data. Current approaches typically rely on a static network decomposition, e.g., personalizing only the final classification layer without adapting to differing degrees of heterogeneity across clients or across network layers. Conversely, aggressively expanding the personalized portion of the model can substantially increase the risk of overfitting. A principled mechanism for navigating this trade-off remains absent from the literature. Furthermore, the training process is generally agnostic to errors made in earlier model iterations, limiting its ability to correct persistent residual errors, a limitation that is particularly severe in non-IID federated settings.

\subsection{Our Contribution}

We introduce a novel boosting-inspired framework for personalized federated learning, called \emph{Progressive Personalized Federated Ensembles} ({PPFE}), which addresses both aforementioned limitations. Our method trains an ensemble of models over $T$ stages, where each stage progressively increases the degree of personalization by shifting additional layers from the global encoder to the client-specific head. To mitigate overfitting, we constrain the complexity of the personalized components using either width shrinkage or low-rank approximations. Moreover, each model is trained sequentially in a boosting-like manner: residual errors from previous stages are reweighted and explicitly targeted by subsequent models. This results in a powerful per-client ensemble that integrates representations from increasingly personalized submodels. Owing to knowledge transfer across stages and progressive task simplification, later stages converge more rapidly. The PPFE method enjoys several advantages: (i) it adaptively tailors model complexity to local data while preserving generalizable global representations; (ii) it sequentially emphasizes hard-to-fit examples through sample reweighting; and (iii) it achieves state-of-the-art performance on multiple benchmark datasets spanning both vision and NLP tasks in heterogeneous federated learning. The proposed approach is simple to implement, incurs (at worst) negligible additional computational/communication overhead compared to standard methods, and is fully compatible with existing federated optimization routines.

We also establish a series of theoretical guarantees, which subsume several prior methods as special cases, most notably FedRep \cite{collins2021exploiting}. Consider a federated network with $K$ clients, where client $k$ holds $n_k$ samples drawn from a private data-generating distribution $P_k$. The distributions $\{P_k\}_{k=1}^K$ may be highly heterogeneous, corresponding to the non-IID setting. Clients are connected to a central server and collaboratively train models via federated learning. We define the \emph{generalization gap} as the difference between the empirical performance achieved by our proposed algorithm in the sample-limited federated setting and the optimal performance attainable in the idealized regime where all $n_k \to \infty$ and all data are centrally available. For our method (Algorithm~\ref{alg:ppfe} in Section~\ref{sec:proposed}), we prove that, with high probability, the generalization gap is bounded as
\begin{align}
\leq 
\left(
\sum_{t=1}^{T}
\frac{\mathsf{cap}(\mathrm{Common}^{(t)})}{K\bar{n}}\right)^{1/2}
+
\left(
\sum_{t=1}^{T}
\frac{\mathsf{cap}(\mathrm{Personal}^{(t)})}{\bar{n}}\right)^{1/2}
+
\left(\frac{T}{\bar{n}}\right)^{1/2},
\label{eq:intro:genGapBound}
\end{align}
where $\bar{n}$ denotes the average of $\{n_k\}_{k=1}^K$, i.e., the average local sample size, and thus $K\bar{n}$ corresponds to the total effective sample size across all clients. The terms $\mathrm{Common}^{(t)}$ and $\mathrm{Personal}^{(t)}$ denote the shared and personalized components of the model at boosting stage $t$, respectively. The \emph{capacity} function $\mathsf{cap}(\cdot)$ could be any generic learning-theoretic complexity measure, such as VC dimension or metric entropy (e.g., logarithmic covering numbers). This bound shows that the potentially large complexity of the shared model components is mitigated by the total effective sample size across clients, while the complexity growth induced by boosting over $T$ stages remains controlled. Moreover, the contribution of the personalized components is explicitly quantified and can be regulated through complexity control. A detailed discussion of these effects, along with specialization of the bound to deep neural networks, is provided in Section \ref{sec:theory}, in particular equation \eqref{eq:theory:GenInformalBound}, Theorem \ref{thm:main:informal} and Corollary \ref{corl:specialSetting:informal}. Full proofs and derivations appear in \ref{sec:app:theory}.

The remainder of the paper is organized as follows. Related work is reviewed in Section~\ref{sec:literature}, and notation is introduced in Section~\ref{sec:notation}. Section~\ref{sec:proposed} presents the proposed method and Algorithm~\ref{alg:ppfe} (PPFE), together with discussions of convergence, computational cost, and communication complexity. Theoretical results are developed in Section~\ref{sec:theory}. Experimental evaluations on synthetic and real-world datasets, including vision and NLP tasks, are reported in Section~\ref{sec:experiments}. Conclusions are drawn in Section~\ref{sec:conclusion}.

\section{Related Works}
\label{sec:literature}

\subsection{Data Heterogeneity in Federated Learning }
Federated learning was introduced with the FedAvg algorithm \cite{mcmahan2017communication}, which aggregates locally trained client models into a single global model. It assumes IID client data, but its performance degrades significantly under non-IID conditions. Subsequent work has focused on improving performance under data heterogeneity. FedProx \cite{li2020fedprox} added a proximal term to the training objective function in order to prevent the local model from moving away too far from the initial aggregated model. SCAFFOLD \cite{karimireddy2020scaffold} proposed a control variate to correct the drift stemming from local updates. FedNova \cite{wang2020tackling} performed gradient‐normalized aggregation, which eliminates objective drift from unequal local training. MOON \cite{li2021model} augmented each client’s local update with a model-level contrastive loss to improve representation learning. FedDyn \cite{Acar2021FederatedLB} and FedDC \cite{gao2022feddc} used regularization terms to control the global-local parameter gap in addition to correcting the client drift. FedDF \cite{lin2020ensemble} and Fed-ET \cite{Cho2022HeterogeneousEK} replaced weight averaging with knowledge distillation, ensembling clients’ logits on a shared unlabeled dataset to train a global student model. FFGB \cite{shen2022federated} executed multiple restricted functional-gradient steps locally, which fit a weak learner to each client’s residual, and then aggregate these weak learners into a global additive model. FedProc \cite{mu2023fedproc} aggregates class prototypes and applies a prototypical contrastive loss to align local embeddings with their global prototypes, repel other classes, and reduce client drift. Fed-ensemble \cite{shi2023fed} and FedEL \cite{wu2024fedel} address data heterogeneity by constructing ensembles of global models and aggregating their predictions at inference time, which improves robustness when client data distributions differ. TCT \cite{yu2022tct} proposes a multi-step approach: train feature representations, freeze them, and solve the remaining task via bootstrapped neural tangent kernel regression. This results in a convex objective, enabling reliable convergence under non-IID data.

\subsection{Model Personalization in Federated Learning}
In contrast, this paper focuses on PFL, which aims to learn client-specific models to better address data heterogeneity. Under non-IID data distributions, enforcing a single global model is inherently restrictive and often suboptimal. PFL alleviates this limitation by enabling the construction of customized models for each client, thereby simultaneously leveraging shared global knowledge through federation and adapting effectively to users' limited local data via personalization. Popular PFL approaches include fine-tuning, clustering, multi-task learning, meta learning, using regularized loss functions, and decoupling the model into global and local parts. 

Fine-tuning in PFL simply starts by learning a global model, and then fine-tuning it for several epochs locally; see, for example, \cite{cheng2021fine} and \cite{Chen2022OnBG}.
Clustering-based approaches, such as those proposed in \cite{cho2023communication} and \cite{mansour2020three}, group clients with similar data distributions, restrict communication within each cluster, and train a separate model for each group. Multi-task federated learning treats every client as a separate task, learning common features and task-specific tweaks so each client keeps a personalized model while still benefiting from external data. Another personalized FL method, MOCHA \cite{smith2017federated}, jointly learns each client’s model and a task-similarity matrix. In this way, the devices share a structure, yet retain tailored predictors. FedU \cite{dinh2021fedu} links client models through a graph-Laplacian regularizer, enabling knowledge sharing during training while still producing personalized models. Meta-learning methods, such as \cite{fallah2020personalized}, \cite{chen2018federated}, and \cite{acar2021debiasing} in the PFL, aim to train a global meta-model that captured transferrable priors through a Model-Agnostic Meta-Learning (MAML) \cite{finn2017model} approach. This facilitates each client to obtain a personalized model with just a few local gradient steps, achieving quick adaptation under non-IID data.

In other approaches such as \cite{zhangpersonalized}, \cite{deng2020adaptive}, \cite{hanzely2020federated}, and \cite{marfoq2022personalized}, each client incorporates the global shared model with a local model to balance the global knowledge with its own personalization. Regularization-based PFL like pFedMe \cite{t2020personalized} and Ditto \cite{li2021ditto}  kept a shared anchor model on the server and let every client minimize their own loss plus a proximity term, so that each client learns a personalized model that remains close enough to the anchor to exploit cross-client knowledge and control the distance with a proximal coefficient hyper parameter. Beyond regularization-based personalization, recent methods have explored objective-level decoupling to better separate global knowledge transfer from local adaptation. For instance, DKD-pFed \cite{su2025dkd} leverages knowledge distillation together with feature decorrelation to mitigate client interference and enhance personalization under non-IID data.

Centralized learning suggests that early layers (the body) capture general features, while later layers (the head) specialize in task-specific outputs. This successful decoupling paradigm has motivated several adaptations in PFL. FedPer \cite{arivazhagan2019federated} trains both the body and head on each client but aggregates only the body, allowing clients to replace their local body with the global one while keeping a private head. LG-FedAvg \cite{liang2020think} follows a similar idea but reverses the structure: clients share their head during aggregation via FedAvg, while keeping the body personalized. FedRep \cite{collins2021exploiting} decouples training into two phases: first, it trains the head locally; then, it freezes the head and updates only the body, which is subsequently aggregated. FedBABU \cite{Oh2021FedBABUTE} takes the opposite approach: it trains only the body during local updates, keeping the head fixed, and later fine-tunes the head locally after body aggregation. FedPAC \cite{Xu2023PersonalizedFL} extends this by sharing a global feature extractor, aligning local embeddings to global prototypes, and blending peer classifier heads for final prediction. pFedFDA \cite{mclaughlin2024personalized} learns a shared feature extractor together with a global generative classifier,
then personalizes by adapting the global feature distribution to each client’s local feature distribution.

Our method (PPFE) shares the most similarities with the latter type of personalization scheme. Unlike the other decoupling methods, we gradually change tasks to focus more on personal information. Simultaneously, to avoid overfitting, we substitute the head model with a less complex model and increase the share of personalized parameters. Finally, each client has trained several models with different complexity and personalization ratio of the model and use an ensemble of them for final result.

\section{Preliminaries}
\label{sec:notation}

For $n\in \mathbb{N}$, let $[n]$ denote the set $\left\{1,\ldots,n\right\}$. Let $\mathcal{X}$ and $\mathcal{Y}$ be two measurable spaces. We denote by $\mathcal{X}$ the input feature space (usually with $\mathcal{X}\subseteq\mathbb{R}^{d}$ for some ambient dimension $d\in\mathbb{N}$), and by $\mathcal{Y}$ the label space. In a multi-class classification problem with $C\ge 2$ output classes we have $\mathcal{Y}=[C]$, while for regression tasks we have $\mathcal{Y}=\mathbb{R}$. Our framework is agnostic to the choice of $\mathcal{Y}$, and thus works for various types of supervised tasks. Throughout the paper, we use $\widehat{A}$ to denote the \emph{empirical/sample} expectation of a random variable $A$.

Let us consider a non-IID federated learning setup with $K\in\mathbb{N}$ clients, each engaged in a supervised learning task (e.g., classification or regression). Each client $k \in [K]$ holds a local dataset $\mathcal{D}_k = \{ (\boldsymbol{X}_i^{(k)}, y_i^{(k)}) \}_{i=1}^{n_k}$ consisting of $n_k$ samples drawn independently from a local distribution $P_k$ over $\mathcal{X} \times \mathcal{Y}$. We let the distributions $P_1,\ldots,P_K$ to be significantly heterogeneous across clients. It should be noted that no one knows about $P_k$s, while $\mathcal{D}_k$ is known only by client $k$. Each client, through communication with a central server, aims at training a local model $h^*_k:\mathcal{X}\to\mathcal{Y}$, which results into a small value to the true (population) risk
$
\mathcal{L}_k(h^*_k)
\triangleq
\mathbb{E}_{(\boldX,y)\sim P_k}
\left[\ell(y,h^*_k(\boldX))\right],
$
where $\ell:\mathcal{Y}\times\mathcal{Y}\to\mathbb{R}_{\ge0}$ is an arbitrary and prespecified loss function (e.g., cross-entropy loss, square loss, etc.). However, client $k$ can only compute its \emph{empirical} loss based on $\mathcal{D}_k$, which is 
\begin{equation}
\widehat{\mathcal{L}}_k(h)
\triangleq
\frac{1}{n_k}
\sum_{i=1}^{n_k}\ell\bigl(y^{(k)}_i,
h(\boldX^{(k)}_i)
\bigr),
\nonumber
\end{equation}
for any given model $h\in\mathcal{H}$. Minimizing $\widehat{\mathcal{L}}_k$ results into the well-known alternative approach of \emph{Empirical Risk Minimization} (ERM). For the hypothesis class $\mathcal{H}$, we usually consider deep neural networks. Suppose a $L$-layer neural network model with weights denoted by $\boldsymbol{W} = (W_0, W_2, \dots, W_{L-1})$ and $\boldsymbol{b}=(b_1,\ldots,b_L)$, where $W_l$ represents a $N_{l}\times N_{l+1}$ weight matrix and $b_{l+1}$ a $N_{l+1}$-dimensional bias vector. Hence, the total number of parameters is denoted by $D = \sum_{l=0}^{L-1} N_{l}N_{l+1}+N_{l+1}$. In this work, and due to using boosting-inspired strategies, we sometimes use \emph{weighted} variants of empirical loss over the training dataset $\mathcal{D}_k$, i.e.,
\begin{align}
\widehat{\mathcal{L}}_k(h; \boldsymbol{\omega}_k)
\triangleq
\sum_{i=1}^{n_k}
\frac{\omega_{k,i}}{n_k}
\ell\bigl(y_i^{(k)},h(\boldsymbol{X}_i^{(k)})\bigr),
\label{eq:empRiskDef}
\end{align}
where $\boldsymbol{\omega}_k \in \mathbb{R}^{n_k}$ is an optional weight vector satisfying $\omega_{k,i} \ge 0$ for all $i$ and $\sum_{i} \omega_{k,i} = n_k$. When the weight vector $\boldsymbol{\omega}_k$ is omitted from $\mathcal{L}_k$, we mean uniform weights, i.e., $\omega_{k,i} = 1$ for all $i$. We leverage the flexibility of the weight vectors $\boldsymbol{\omega}_k$ in later stages of the method, where our goal is to introduce ensemble mechanisms.

\section{Proposed Method}
\label{sec:proposed}
Standard federated learning algorithms, such as FedAvg \cite{mcmahan2017communication}, aim to learn a single global model by minimizing the average empirical risk across all clients as follows:
\begin{align}
\theta^*_{\mathrm{FedAvg}} = \argmin_{\theta \in \Theta} 
\frac{1}{K}\sum_{k=1}^K 
\widehat{\mathcal{L}}_k(h_{{\theta}}),
\label{eq:fedavg}
\end{align}
where $\Theta\subseteq\mathbb{R}^D$ denotes the space of all parameters,
and $h_{\theta}:\mathcal{X}\to\mathcal{Y}$ is the predictive function (e.g., classifier or regressor) associated with the weights $\theta$. In heterogeneous settings, where client distributions differ significantly, training a single shared model often yields suboptimal performance. 
A widely adopted approach in personalized FL decomposes the neural network $h_{\theta}(\cdot)$ into two components: a \emph{shared representation extractor} $f_\psi : \mathcal{X} \to \mathbb{R}^{d_r}$, parameterized by shared weights $\psi \in \Psi$, and a \emph{client-specific head} $g_{\phi_k} : \mathbb{R}^{d_r} \to \mathcal{Y}$, parameterized by $\phi_k \in \Phi$ for client $k$. The prediction model for client $k$ is then given by the composition $(g_{\phi_k} \circ f_\psi)(\cdot)$.

The shared representation $f_\psi$ maps inputs to a common feature space of dimension $d_r \in \mathbb{N}$ and typically corresponds to the earlier layers of the network, e.g., $(W_1, \dots, W_{L'})$ for some $1 \le L' < L$. The task-specific head $g_{\phi_k}$ maps these representations to outputs and is more sensitive to local distributional shifts, typically corresponding to the remaining layers $(W_{L'+1}, \dots, W_L)$. This decomposition enables collaborative learning of a shared representation while preserving the flexibility of client-level personalization. Here, $\Psi$ and $\Phi$ denote the sets of all possible weights for the shared and personalized components, respectively, which may be further constrained by user-defined regularizations. The mappings $f_\psi$ and $g_{\phi_k}$ denote the functions corresponding to weights $\psi$ and $\phi_k$. We adopt the notations $f$ and $g$ to clearly distinguish between the shared and personalized parts of the network. Thus, the corresponding optimization objective becomes:
\begin{align}
\left(
\widehat{\psi}^*, \widehat{\phi}_1^*, \ldots, \widehat{\phi}_K^* \right)
\leftarrow 
\min_{\psi \in \Psi}~
\frac{1}{K}\sum_{k=1}^K 
\left( \min_{\phi_k \in \Phi}~\widehat{\mathcal{L}}_k(g_{\phi_k}\bigl(f_\psi(\cdot)\bigr) \right).
\end{align}
Similar to FedAvg, this objective is minimized in a federated manner with the participation of all clients. As noted in Section \ref{sec:intro}, this approach has inherent limitations. We propose to extend it using a boosting-inspired framework, which improves upon prior methods in two principal ways:
\noindent
\\[1mm]
\textbf{Progressive Personalization:} Instead of a single model, we train a sequence of $T$ models, indexed by $t = 1, 2, \dots, T$, with progressively increasing levels of personalization. This is achieved by gradually shifting more layers from the shared feature extractor to the client-specific head. In parallel, we restrict the complexity of the personalized layers by restricting the \emph{rank} of weight matrices as $t$ is increased. This constraint helps mitigate overfitting, especially since the head $g_{\phi_k}$ is trained solely on the local dataset of client $k$, in contrast to the feature extractor $f_\psi$, which benefits from multi-client data and exhibits lower variance.
\noindent
\\[1mm]
\textbf{Federated Boosting:} Rather than training all $T$ models independently or in parallel, we adopt a sequential training procedure. Each model in the sequence is trained to correct the residuals of its predecessors, following a boosting-like strategy. This design allows later models to focus on hard-to-fit aspects of the local data that earlier models failed to capture.


\subsection{Progressive Personalization}
\label{subsec-progressive-personalization}

Building on empirical evidence from \cite{yu2022tct}, we observe that \emph{global} information tends to be captured by the early layers of deep neural networks, whereas \emph{local, client-specific} patterns typically emerge in the deeper layers. Consequently, the optimal split between shared and personalized components may vary across clients, depending on sample size and data distribution. A simple baseline such as FedRep \cite{collins2021exploiting} personalizes only the final classification layer, keeping all earlier layers shared. However, in highly heterogeneous settings, the assumption of a universally shared representation across clients often breaks down. This can result in misaligned feature spaces and compel clients to learn brittle, overfit decision boundaries.

\begin{figure}[!t]
\centering
    \includegraphics[trim={18cm 7cm 0 7cm},clip,width=\textwidth]{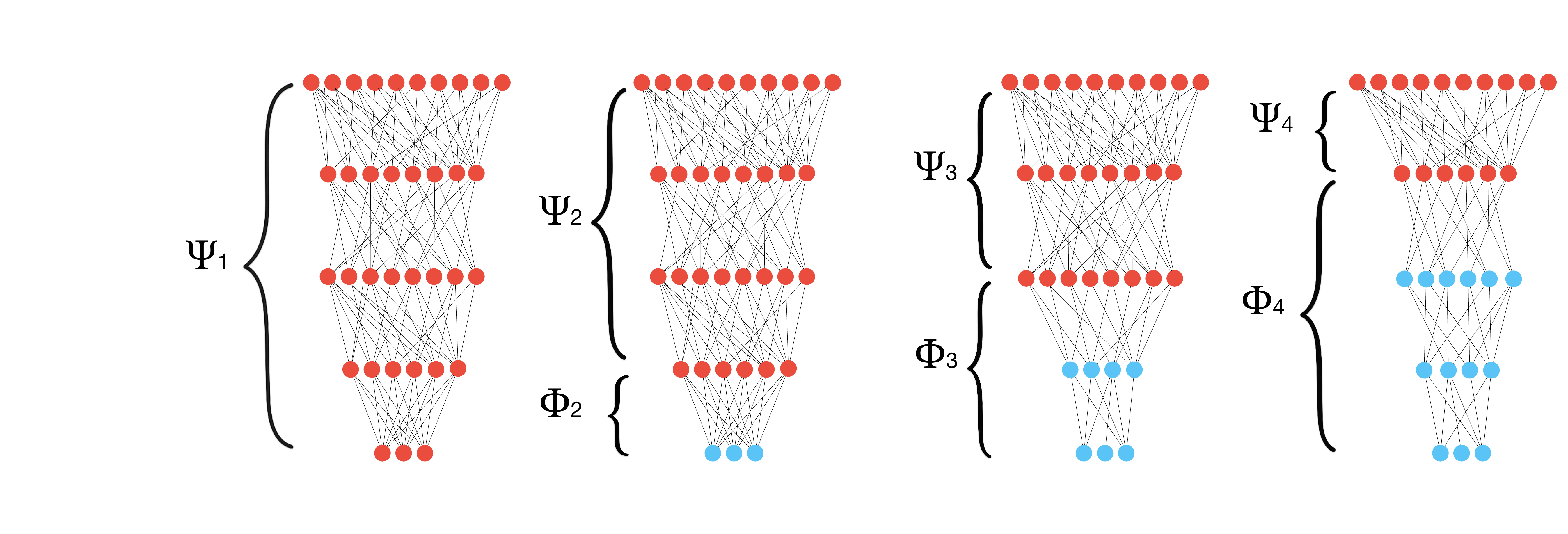}
    \caption{
    Progressive personalization procedure, transitioning from a fully global model to increasingly personalized models across successive stages (for $t=1,2,3$ and $4$). As can be seen, we either shrink the personalized layer widths, or lower their effective parameter size via low-rank decomposition.}
    \label{fig:prog_personal}
\end{figure}

Our proposed scheme generalizes this idea: by progressively decoupling the deeper layers of the network, each client is empowered to adapt its feature extractor to better match its local data distribution, while still benefiting from the globally learned representations in earlier layers. Specifically, we expand the personalized portion of the network over $T$ stages, ranging from $t = 1$ (no personalization; equivalent to FedAvg) to $t = T$ (maximum allowed personalization). This allows each client to gradually increase its modeling capacity in response to data heterogeneity. Due to the limited and highly non-IID data available to each client, allowing overly complex personalized heads may hurt generalization. To avoid local overfitting as more layers are personalized, we control model complexity via structural constraints such as low-rank approximations on the weight matrices of the personalized heads or simply shrinking their width. Formally, suppose that for each client $k \in [K]$, we construct an ensemble model $\{ (\psi^{(t)}, \phi^{(t)}_{k}) \}_{t=1}^{T}$, where $\psi^{(t)}$ denotes the shared (global) parameters at stage $t$, and $\phi^{(t)}_{k}$ denotes the client $k$s specific head. At the initial stage ($t = 1$), we train a fully global model using FedAvg:
\begin{align}
\widehat{\psi}^{*(1)} = \argmin_{\psi \in \Psi_1}
\frac{1}{K}\sum_{k=1}^{K} 
\widehat{\mathcal{L}}_k\bigl(f_\psi(\cdot)\bigr),
\end{align}
where $\Psi_1 = \mathbb{R}^D$ is the space of all model parameters. For $t > 1$, we progressively personalize the final $t-1$ layers of the network. Let $\Psi^{(t)}$ denote the parameter space of the shared layers $(W_0, \ldots, W_{L-t})$, and $\Phi^{(t)}$ denote the parameter space of the personalized tail $(W_{L-t+1}, \ldots, W_{L-1})$, subject to a user-defined low-rank constraint. Then, Figure \ref{fig:prog_personal} shows a depiction of our approach. Also, the objective at stage $t$ becomes:
\begin{align}
\label{eq:vanillaERM}
\bigl(\widehat{\psi}^{*(t)}, \widehat{\phi}_{1:K}^{*(t)}\bigr) 
\leftarrow 
\min_{\psi \in \Psi^{(t)}} ~
\frac{1}{K}\sum_{k=1}^{K} 
\left( \min_{\phi_k \in \Phi^{(t)}}~
\widehat{\mathcal{L}}_k\bigl(
g_{\phi_k}(f_\psi(\cdot)
\bigr) \right).
\end{align}
This ``deeper-but-slimmer'' strategy ensures that each additional personalized layer does not inflate the total model capacity, thus improving generalization. Also, model diversity across the ensemble stems primarily from where personalization occurs (i.e., which layers), rather than from differences in overall size or capacity, an essential feature for boosting, as discussed in the next subsection.

\begin{figure}[!t]
    \centering
    \includegraphics[width=0.8\columnwidth]{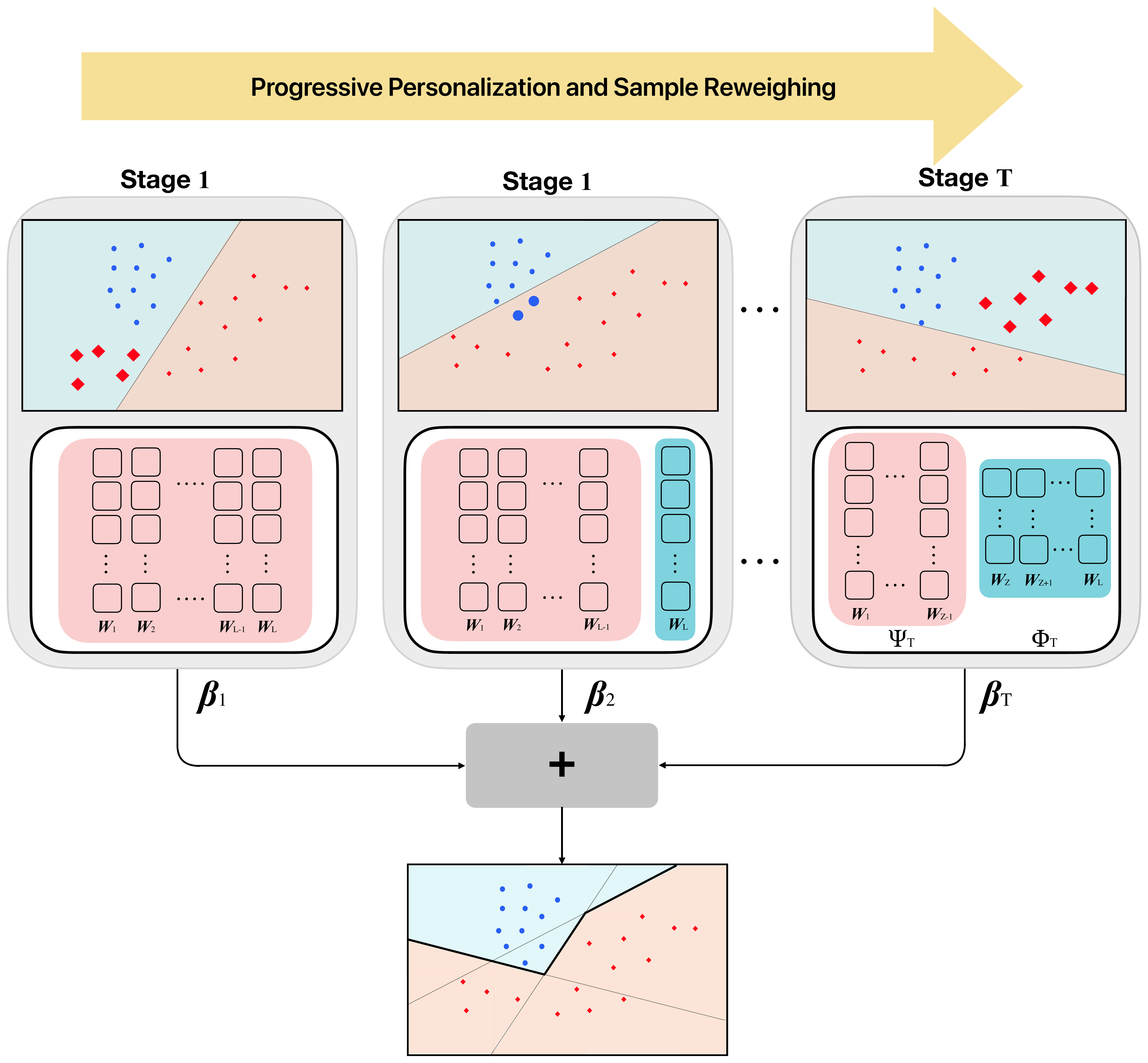}
    \caption{
    Overview of the proposed PPFE method. The number of personalized layers is progressively increased across stages. At each stage, samples are reweighted using feedback from the model of the previous stage, and the final model is an ensemble of the models from all stages.}
    \label{fig:ppfe}
\end{figure}

The adopted low-rank matrix factorization is via Singular Value Decomposition (SVD), following the method of \cite{denton2014exploiting}. Given a fully connected layer with weight matrix $W_l \in \mathbb{R}^{N_l \times N_{l+1}}$, we compute its SVD as $W_l = U \Sigma V^\top$ and retain only the top-$r$ singular components to obtain a low-rank approximation $U_{(:,1:r)} \Sigma_{(1:r,1:r)} V^\top_{(1:r,:)}$, with $r \ll \min(N_l, N_{l+1})$. This yields a compressed model with a parameter count reduced by a factor of
$\frac{N_\ell N_{\ell+1}}{r(N_\ell + N_{\ell+1})}$,
while preserving dense matrix operations and offering a guaranteed low approximation error in Frobenius norm \cite{tai2015convolutional}. An alternative strategy is to directly reduce model complexity by shrinking the widths $N^{(t)}_l$ of the personalized layers at iteration $t$, as illustrated in Figure \ref{fig:prog_personal}. Both approaches lead to comparable theoretical generalization bounds (see Section \ref{sec:theory}) and exhibit similar empirical performance in practice.


\subsection{Federated Boosting}
Training all the $T$ models in parallel often leads to inefficiencies. On the other hand, a sequential training approach provides several benefits:\
(i) Layer-wise initialization using weights from earlier stages mitigates catastrophic forgetting and facilitates more stable optimization. 
(ii) Iterative reweighting of training samples allows later models to focus on correcting residual errors from previous ones. 
(iii) In low-data regimes, freezing too many layers prematurely can hinder convergence, since deep heads trained from scratch often yield poor representations.

\begin{algorithm}[!t]
\caption{\small{Personalized Progressive Federated Ensemble (PPFE)}}
\label{alg:ppfe}
\KwIn{Clients $\{1,\dots,K\}$ with datasets $\{\mathcal{D}_k\}_{k=1}^{K}$, ensemble size $T$, parameter spaces $\{(\Phi^{(t)}, \Psi^{(t)})\}_{t=1}^{T}$, participation ratio $\rho$, rounds per iteration $\tau_0$}
\KwOut{Personalized ensemble models $\{h_k^{\mathrm{ens}}\}_{k=1}^K$}

\SetKwFunction{Reweight}{Reweight}
\SetKwProg{Fn}{Function}{:}{}

\Fn{\Reweight{$f(\cdot), \mathcal{D}, \boldsymbol{\omega}$}}{
    $\varepsilon \gets 
    \dfrac{\sum_i \omega_i\,\ell(y_i,f(\mathbf{X}_i))}{(\max~ \ell(\cdot))\cdot\sum_i \omega_i}$,\quad
    $\beta \gets \frac{1}{2}\log\!\left(\frac{1-\varepsilon}{\varepsilon}\right)$,
    
    \For{$i\gets 1$ \KwTo $\mathrm{dim}(\boldsymbol{\omega})$}{
        $\widehat{\omega}_i \gets \omega_i \cdot\exp\!\left(\beta\,\ell(y_i,f(\mathbf{X}_i))\right)$
    }
    Normalize $\widehat{\boldsymbol{\omega}}$ so that
    $\sum_i \widehat{\omega}_i = \mathrm{dim}(\boldsymbol{\omega})$
    
    \Return $(\beta, \widehat{\boldsymbol{\omega}})$
}

Train $f_{\psi^{(1)}}$ via FedAvg over all clients

\For{$k\gets 1$ \KwTo $K$}{
    $(\beta_k^{(1)}, \boldsymbol{\omega}_k^{(2)}) \gets$
    \Reweight{$f_{\psi^{(1)}}, \mathcal{D}_k, \boldsymbol{\omega}_k^{(1)}=(1,\dots,1)$}
}

\For{$t\gets 2$ \KwTo $T$}{
    Initialize shared components $\widehat{\psi}_{[0]} \gets \psi^{(t-1)}$,
    
    Initialize local heads $\widehat{\phi}_{k,[0]} \gets \phi_k^{(t-1)}$ for all $k\in[K]$.

    \For{$\tau\gets 1$ \KwTo $\tau_0$}{
        Sample client subset $\mathcal{S}_\tau\subseteq[K]$ of size $\lceil \rho K\rceil$
        
        \ForEach{$k\in\mathcal{S}_\tau$}{
            $(\widehat{\phi}_{k,[\tau]}, \widehat{\psi}_{k,[\tau]}) \gets$ LocalTraining$(\widehat{\psi}_{[\tau-1]},\widehat{\phi}_{k,[\tau-1]},\mathcal{D}_k,\boldsymbol{\omega}_k^{(t)})$
            
            Send $\widehat{\psi}_{k,[\tau]}$ to server
        }
        Server aggregates $\widehat{\psi}_{[\tau]} \gets
        \frac{1}{|\mathcal{S}_\tau|}\sum_{k\in\mathcal{S}_\tau}\widehat{\psi}_{k,[\tau]}$, and sends back $\widehat{\psi}_{[\tau]}$
    }

    Set $\psi^{(t)} \gets \widehat{\psi}_{[\tau_0]}$, and set $\phi_k^{(t)} \gets \widehat{\phi}_{k,[\tau_0]}$ for all $k$

    \For{$k\gets 1$ \KwTo $K$}{
        $(\beta_k^{(t)}, \boldsymbol{\omega}_k^{(t+1)}) \gets$
        \Reweight{$\sum_{u=1}^t\beta_k^{(u)}g_{\phi_k^{(u)}}(f_{\psi^{(u)}}(\cdot)),\;\mathcal{D}_k,\;\boldsymbol{\omega}_k^{(t)}$}
    }
}

\For{$k\gets 1$ \KwTo $K$}{
    $h_k^{\mathrm{ens}}(\cdot)\gets\sum_{t=1}^T
    \beta_k^{(t)}g_{\phi_k^{(t)}}(f_{\psi^{(t)}}(\cdot))$
}

\end{algorithm}

Therefore, we leverage optional per-sample weight vectors $\boldsymbol{\omega}_k$ in the weighted empirical loss $\widehat{\mathcal{L}}_k(\cdot; \boldsymbol{\omega}_k)$ (defined in \eqref{eq:empRiskDef}), and replace the vanilla ERM outlined in \eqref{eq:vanillaERM} with a weighted ERM. These weights are updated at each iteration $t$ based on the performance of earlier models. At the initial stage ($t=1$), the global model is trained via FedAvg, with each client assigning uniform weights: $\omega^{(1)}_{k,i} = 1$. At stage $t > 1$, each client $k$ initializes its model with the previous parameters, reweights samples via $\omega^{(t)}_{k,i}$ (for $i\in[n_k]$) based on residual errors, similar to the AdaBoost algorithm \cite{schapire2013explaining}. Specifically, the update rule for weights is:
\begin{align}
\omega^{(t)}_{k,i}
\propto
\omega^{(t-1)}_{k,i}
\cdot
\exp\left[-\gamma
\ell\!\left(
y^{(k)}_i,
g_{\phi^{(t-1)}_{k}}\bigl( 
f_{\psi^{(t-1)}}
\bigl(\boldsymbol{X}^{(k)}_i\bigr)
\bigr)
\right)
\right],
\quad
\forall i\in[n_k],
\end{align} 
where $\gamma > 0$ is an inverse temperature parameter defined later in Algorithm~\ref{alg:ppfe}. The weights are then normalized to satisfy $\sum_{i=1}^{n_k} \omega^{(t)}_{k,i} = n_k$ for all $k$. We also compute ensemble coefficients $\beta^{(t)}_k$ that reflects the predictive accuracy of the model at iteration $t$ for client $k$. After $T$ iterations, we obtain shared backbone models $\psi^{(t)}$ for $t\in[T]$ and personalized heads $\phi^{(t)}_{k}$. The final ensemble for client $k$ is then defined as:
\begin{align}
h^{(\mathrm{ens})}_k(\cdot)
\triangleq
    \sum_{t=1}^{T} 
    \beta^{(t)}_k\!
    \left(
    g_{\phi^{(t)}_{k}} \circ f_{\psi^{(t)}}
    \right)(\cdot),
    \quad \forall k\in[K].
\end{align}

Our proposed algorithm, Progressive Personalized Federated Ensembles (PPFE), is detailed in Algorithm \ref{alg:ppfe} and we also depicted it in the Figure \ref{fig:ppfe}. In each communication round, the server selects a small subset of clients of size $\lceil \rho K\rceil$ (for some $\rho < 1$) and performs federated training among them for $\tau_0$ steps. The average of the aggregated global parameters from this subset is then used to initialize the global model in the next round, while each client’s local parameters are restored from their previous local training state.


\subsection{Convergence, Computational and Communication Complexity}
\label{sec:proposed:complexity}

\textbf{Convergence.}
Our method introduces no additional hard-to-converge components beyond those present in standard baselines such as FedRep \cite{collins2021exploiting}. Specifically, aside from local neural network training via SGD, which is common to essentially all FL and PFL approaches, all algorithm-specific operations in Algorithm \ref{alg:ppfe} (e.g., weight updates and ensemble coefficient computation) have deterministic termination times and incur negligible overhead relative to the cost of training deep neural networks. Moreover, the proposed algorithm is staged: at each stage, the global component is initialized from the solution of the previous stage rather than being retrained from scratch, consistent with prior work. This warm-start strategy empirically accelerates convergence.

\textbf{Train-time and Communication Complexity.}
Although our approach trains $T$ models per client rather than a single model, we ensure a fair comparison by keeping the total training time and communication budget equal to those of non-boosting baselines (see Section \ref{sec:experiments}). Concretely, the time and communication budget allocated to each stage is at most $1/T$ of the total budget used by competing methods. The overall communication budget is distributed across stages such that the total number of communication rounds matches a fixed reference horizon shared by all methods. Consequently, our approach matches the wall-clock training time of prior methods, while potentially reducing communication overhead, since the number of global parameters decreases monotonically across stages.

\textbf{Test-time Complexity.}
At inference, predictions are formed using the collection of stage-wise models, which increases computational cost. In the worst case, this requires $T$ forward passes. In practice, however, the cost is typically lower because later-stage models are lighter and involve fewer global parameters.

\section{Theoretical Guarantees}
\label{sec:theory}

This section is devoted to the analysis of both the \emph{bias} and the \emph{generalization} gap of our proposed method. The generalization gap $\mathsf{Gen}$ is defined as the difference between the minimum empirical training loss attained by the proposed method (Algorithm~\ref{alg:ppfe}) and the minimum achievable population risk in an idealized setting where all data are fully centralized and the sample sizes are infinite. A formal definition is given in \eqref{eq:app:theory:genGapDef} of \ref{sec:app:theory}. For intuition, we may express this quantity informally as
\begin{align}
\mathsf{Gen}
\triangleq
\mathrm{EmpiricalRisk}\!\left(
\mathrm{Algorithm}~\ref{alg:ppfe}
\right)
-
\inf_{\big\{\psi^{(t)},\,\phi^{(t)}_k,\,\beta^{(t)}_{k}\big\}}
\mathbb{E}\!\left[
\mathrm{Risk}\bigl(\psi^{(t)},\phi^{(t)}_k,\beta^{(t)}_{k}\bigr)
\right],
\end{align}
where the expectation is taken with respect to the underlying data-generating distributions $P_1,\ldots,P_K$. Here, $\mathbb{E}[\mathrm{Risk}(\psi^{(t)},\phi^{(t)}_k,\beta^{(t)}_{k})]$ denotes the population (true) risk associated with the choice of shared models $\psi^{(t)}$, personalized models $\phi^{(t)}_k$, and boosting weights $\beta^{(t)}_{k}$, evaluated under the centralized infinite-sample regime. In particular, we study the behavior of $\mathsf{Gen}$ as a function of the network size $K$, the client sample sizes $n_1,\ldots,n_K$, the boosting budget $T$, and the complexities of the local hypothesis classes $\Phi^{(t)}$ and the (more expressive) global hypothesis classes $\Psi^{(t)}$, for all $t \in [T]$. Throughout, we denote by $\bar n$ the harmonic mean of $n_1,\ldots,n_K$, namely,
\begin{equation}
\bar n \;\triangleq\; \frac{K}{\frac{1}{n_1} + \cdots + \frac{1}{n_K}}.
\nonumber
\end{equation}

For a function class $\Psi$---for instance, the class of all linear models, or a family of deep neural networks with a pre-specified architecture---we  denote by $\mathsf{cap}(\Psi)$ its \emph{capacity} or \emph{complexity}. Standard notions in statistical learning theory, such as the VC dimension in binary classification or, in our analysis in \ref{sec:app:theory}, metric-entropy--based quantities (e.g., logarithmic covering numbers), serve as concrete instantiations of such capacity measures and are typically of the same order. In our setting, the hypothesis class at iteration $t$ decomposes into a global (shared) component drawn from a class $\Psi^{(t)}$ and a personalized (client-specific) component drawn from a class $\Phi^{(t)}$, with
$
\mathsf{cap}\bigl(\Psi^{(t)}\bigr) \gg \mathsf{cap}\bigl(\Phi^{(t)}\bigr).
$
This separation reflects the fact that the shared body is substantially more expressive and parameter-heavy, whereas the personalized components are deliberately chosen to be lightweight in order to facilitate generalization from limited local data.

\subsection{Non-Asymptotic Generalization Bound}

We prove that, with high probability, the generalization gap of Algorithm~\ref{alg:ppfe} satisfies the following bound (see the proof of Theorem~\ref{thm:main} in \ref{sec:app:theory}):
\begin{align}
\label{eq:theory:GenInformalBound}
\mathsf{Gen}
\;\leq\;
\widetilde{\mathcal{O}}
\!\left(
\frac{1}{K \bar n}
\sum_{t=1}^{T}
\mathsf{cap}\bigl(\Psi^{(t)}\bigr)
+
\frac{1}{\bar n}
\sum_{t=1}^{T}
\mathsf{cap}\bigl(\Phi^{(t)}\bigr)
+
\frac{T}{\bar n}
\right)^{1/2}.
\end{align}

The three terms in \eqref{eq:theory:GenInformalBound} correspond to distinct mechanisms with different asymptotic behaviors as $K,\bar n \to \infty$:
\begin{itemize}
\item
$\bigl(\frac{1}{K\bar n}\sum_{t=1}^{T}\mathsf{cap}(\Psi^{(t)})\bigr)^{1/2}$ captures the contribution of the shared hypothesis classes. Although these classes are typically highly expressive, their complexity is amortized over the \emph{total effective sample size} $K\bar n$, which is on the order of the total number of samples across all clients. Consequently, strong generalization can be achieved even when individual clients have limited data.

\item
$\bigl(\frac{1}{\bar n}\sum_{t=1}^{T}\mathsf{cap}(\Phi^{(t)})\bigr)^{1/2}$ quantifies the effect of the personalized components. This term does not benefit from the factor $K$, but remains small in practice since $\mathsf{cap}(\Phi^{(t)}) \ll \mathsf{cap}(\Psi^{(t)})$ and is explicitly controlled as a function of the boosting iteration $t$.

\item
$\left(T/\bar n\right)^{1/2}$ reflects the intrinsic complexity introduced by the boosting procedure itself, a phenomenon that also appears in classical analyses of AdaBoost and related methods.
\end{itemize}

A fully formal version of this bound, in which $\mathsf{cap}(\cdot)$ is instantiated via metric entropy (specifically, logarithmic $\ell_\infty$ covering numbers), is provided in the proof of Theorem~\ref{thm:main} in \ref{sec:app:theory}. We now present a more explicit instantiation of the above result for deep neural networks. Consider networks of fixed depth $L>T$, where the width may vary across layers and boosting iterations. Let $N^{(t)}_l$ denote the width of the network used at iteration $t$ in layer $l$, for $l=0,1,\ldots,L$, with $N^{(t)}_0=d$ and $N^{(t)}_L=1$ for all $t$.

\begin{theorem}[Main Result (informal), see Theorem~\ref{thm:main}]
\label{thm:main:informal}
Assume $T$ deep neural network architectures, each of depth $L>T$, and widths $N^{(t)}_l$. Suppose that all network weights are uniformly bounded and that the activation function is Lipschitz continuous. Then, with probability at least $1-\zeta$ for any $\zeta\in(0,1)$, the generalization gap defined in \eqref{eq:theory:GenInformalBound} satisfies
\begin{align}
&\mathsf{Gen}
\;\leq\;
\mathcal{O}\!\left(
\sqrt{\frac{\log(1/\zeta)}{K\bar n}}
\right)
+
\\
&
\sqrt{\frac{\log(K\bar n)}{\bar n}}
\;
\widetilde{\mathcal{O}}
\!\left[
T
+
\sum_{t=1}^{T}
\left(
\frac{1}{K}
\sum_{l=1}^{L-t}
\left(
N^{(t)}_l N^{(t)}_{l-1}+N^{(t)}_l
\right)
+
\sum_{l=L-t}^{L-1}
\left(
N^{(t)}_l N^{(t)}_{l+1}+N^{(t)}_{l+1}
\right)
\right)
\right]^{1/2}.
\nonumber
\end{align}
The $\widetilde{\mathcal{O}}(\cdot)$ notation hides logarithmic dependencies on $T$, $L$, the weight bounds and Lipschitz constants.
\end{theorem}

To further elucidate the implications of our result, we consider a simplified but representative architectural setting.

\begin{corollary}[Special DNN Architectures (informal), see Corollary~\ref{corl:specialSetting}]
\label{corl:specialSetting:informal}
Assume the setting of Theorem~\ref{thm:main:informal}, and suppose that all clients have equal local sample size $n$. Assume that the width of all shared layers is fixed to a base value $W_b$, while the widths of the personalized layers decrease across iterations according to $W_b t^{-\alpha}$ for some $\alpha\ge 0$. Then, with probability at least $1-\zeta$,
\begin{align}
\mathsf{Gen}
\;\le\;
\widetilde{\mathcal{O}}\!\left(
\left(\frac{LT}{Kn}\right)^{1/2}
+
\left(\frac{\log(1/\zeta)}{Kn}\right)^{1/2}
+
\frac{T^{1-\min(\alpha,1/2)}}{n^{1/2}}
\right),
\end{align}
where the hidden constants depend at most linearly on $W_b$ and logarithmically on the weight bounds and Lipschitz constants.
\end{corollary}

The formal statements and proofs are provided in \ref{sec:app:theory}. This corollary highlights the same structural decomposition observed earlier. The term $(LT/(Kn))^{1/2}$ corresponds to the shared representation and reflects the depth $L$ of the network, whose impact is mitigated by aggregation across clients. The term $T^{1-\min(\alpha,1/2)}/n^{1/2}$ captures the contribution of the personalized components and is independent of $L$. By choosing $\alpha\in(0,1/2)$, the dependence on $T$ can be reduced from linear to sublinear, illustrating how progressively simplifying the personalized components improves generalization.

\begin{remark}[Width Shrinkage vs. Low-Rank Decomposition]
In the proposed method in Section \ref{sec:proposed}, we discuss the use of low-rank decompositions as a mechanism for complexity reduction. In contrast, Corollary~\ref{corl:specialSetting:informal} considers layer-width shrinkage. These two approaches are treated equivalently in our theoretical analysis, since the bounds depend only on the number of \emph{effective} parameters after complexity reduction, rather than on the specific parameterization. Consequently, all stated bounds remain applicable under either strategy.
\end{remark}

\begin{remark}[Non-Vacuous Bounds for DNNs]
The bounds presented in Theorem~\ref{thm:main:informal} and Corollary~\ref{corl:specialSetting:informal} rely on metric-entropy--based capacity estimates for DNNs, which are known to be vacuous in some practical regimes. Alternatively, one may work directly with the more abstract form given in \eqref{eq:theory:GenInformalBound} and substitute $\mathsf{cap}(\cdot)$ with a tighter, problem-dependent complexity measure. Our primary purpose is to elucidate the qualitative interplay between $K$, $\bar n$, and complexity measures rather than to provide numerically sharp constants.
\end{remark}

\begin{remark}[Theoretical Analysis of FedRep \cite{collins2021exploiting}]
By restricting to a single iteration and a fixed-size personalized component, Algorithm \ref{alg:ppfe} reduces to FedRep \cite{collins2021exploiting}. In this regime, all terms in our generalization bounds that depend on $T$ either vanish or can be treated as absolute constants. Consequently, our theoretical analysis directly applies to FedRep. To the best of our knowledge, this provides the first nonlinear generalization guarantee for FedRep, filling a gap that was previously unaddressed in the literature.
\end{remark}


\subsection{Analysis of Bias}

Unlike classical or modern generalization guarantees (e.g., those based on complexity measures such as VC dimension, Rademacher complexity, or covering numbers), there are no global, distribution-free bounds on the \emph{bias}, i.e., the minimum achievable training error in the learning stage. Nevertheless, bias can always be measured and controlled prior to deployment, since it depends solely on empirical performance.

Boosting strategies, including the proposed method, have empirically demonstrated a consistent (often exponential) reduction in bias as the number of boosting rounds $T$ increases. To formalize this intuition, let us focus on the binary classification setting, and model each client’s personalized training procedure as a weak learning algorithm. Specifically, given the current distribution of sample weights, the procedure returns a hypothesis $h^{(t)}(\cdot)$ whose true error $\varepsilon_{k}^{(t)}$ satisfies
\begin{equation}
\varepsilon_{k}^{(t)}
\triangleq \mathbb{P}_{(\boldX_i, y_i)\sim \mathcal{D}_k}\!\left[h^{(t)}(\boldX_i)\neq y_i\right]
\leq \tfrac{1}{2}-\gamma,
\label{eq:theory:bias:gammaWeak}
\end{equation}
for some strictly positive (but typically small) constant $\gamma$. This is precisely the definition of a $\gamma$-weak classifier \cite{mohri2018foundations}, which guarantees the learned binary classifier performs weakly better than random coin tossing. Then, the following theorem applies:

\begin{theorem}[Theorem 6.1 of \cite{mohri2018foundations}]
Under the weak learning assumption in \eqref{eq:theory:bias:gammaWeak} for some $\gamma>0$, the empirical error of the ensemble learned at each client $k\in[K]$ satisfies
\begin{equation}
\mathrm{EmpiricalRisk}\!\left(h_k^{(\mathrm{ens})}(\cdot)\right)
\leq e^{-2\gamma^2 T},
\label{eq:theory:Thm:biasExpReduction}
\end{equation}
for any iteration budget $T\in\mathbb{N}$.
\end{theorem}

The bound in \eqref{eq:theory:Thm:biasExpReduction} implies that, as boosting incorporates additional weak learners (i.e., as $T$ increases), the empirical error decreases at an exponential rate, provided that the weak learning condition $\varepsilon_{k}^{(t)} < \tfrac{1}{2}-\gamma$ holds uniformly across boosting rounds. As discussed earlier, there is no general mechanism to guarantee this condition a priori. However, we empirically observe that, due to the increasing expressivity of the local layers across stages in our proposed method, maintaining a non-negligible $\gamma$ throughout training is easier compared to competing approaches. This phenomenon likely contributes to the superior empirical performance reported in Section \ref{sec:experiments}.

\paragraph{Bias--Variance Tradeoff.}
Combining the generalization bounds in \eqref{eq:theory:GenInformalBound} (or, more concretely, Theorem~\ref{thm:main:informal} and Corollary~\ref{corl:specialSetting:informal}), which predict an increase in the generalization gap on the order of $T^{1-\min(1/2,\alpha)}$, with the exponential bias reduction in \eqref{eq:theory:Thm:biasExpReduction}, suggests the existence of an optimal iteration budget $T$. In particular, choosing a smaller value of $\alpha$ yields more expressive personalized layers in the setting of Corollary~\ref{corl:specialSetting:informal}, which in turn reduces the bias; however, as the bound indicates, this comes at the cost of an increased generalization gap. This observation reveals an explicit and tunable bias--variance tradeoff enabled by our framework, a degree of control that is largely absent in prior methods. In practice, the optimal value of $T$ can be selected either analytically or empirically via cross-validation. Consistent with prior work on boosting, we do not observe strong sensitivity to the choice of $T$ in our experiments.

\section{Experiments}
\label{sec:experiments}
We evaluate PPFE using both synthetic and real-world tasks under federated settings. We first test on a linear regression problem with synthetic data in Section \ref{sec:experiments:synthetic}, then assess our performance on some standard real-world benchmarks via practical neural network architectures in Section \ref{sec:experiments:real}.


\subsection{Synthetic Data}
\label{sec:experiments:synthetic}

We simulate a synthetic setting in which clients share both global and personalized components with varying proportions. Data are generated according to a linear regression model. Specifically, for each client $k \in [K]$, the feature vectors are sampled as
$\boldsymbol{X}^{(k)}_{1:n_k} \iid \mathcal{N}(\boldsymbol{0}, \boldsymbol{\Sigma}_k)$
for a client-specific covariance matrix $\boldsymbol{\Sigma}_k$, and the corresponding responses are generated as
\begin{align}
y^{(k)}_i = \boldsymbol{W}_k^{\top} \boldsymbol{X}^{(k)}_i + \epsilon_i,\quad \forall i\in[n_k],
\end{align}
where the noise terms $\epsilon_i \sim \mathcal{N}(0, \sigma_{\epsilon}^2)$ are drawn IID across all samples and clients. Each client-specific weight vector $\boldsymbol{W}_k$ consists of two components: a global component shared across all clients, sampled once from $\mathcal{N}(\boldsymbol{0}, \sigma_g^2 \boldsymbol{I}_d)$, and a personalized component unique to each client, sampled from $\mathcal{N}(\boldsymbol{0}, \alpha_k \sigma_g^2 \boldsymbol{I}_d)$, where $\alpha_k$ is a client-specific coefficient. The final weight vector is constructed as a convex combination of these components, controlled by a personalization ratio $r_p \in [0,1]$. Concretely, letting $\boldsymbol{W}_g$ denote the shared global component and $\boldsymbol{W}^{(k)}_{l}$ the local component for client $k$, we define
\begin{align}
\boldsymbol{W}_k
&= (1 - r_p)\boldsymbol{W}_g + r_p \boldsymbol{W}^{(k)}_{l},
\qquad \forall k = 1, \dots, K.
\end{align}

To evaluate the behavior of PPFE as the federation scales, we study its performance as the number of clients increases. We fix the personalization ratio $r_p$, input dimension $d$, and per-client sample size $n_k$, and vary the number of clients from $20$ to $200$. PPFE is run for four stages with decreasing complexity parameters $\alpha_t$ at each stage $t$, assuming full client participation and exact local optimization.

At each stage, Ridge regression is used to estimate the weight vectors. To select the regularization parameter $\lambda \ge 0$, we reserve $20\%$ of each client’s training data for validation and perform a grid search. The regularization parameter is progressively reduced across stages, and the final test error is then evaluated. For comparison, we also train (i) local Ridge regression models independently on each client, and (ii) a FedAvg baseline obtained by averaging model parameters across clients, using the same validation procedure to tune $\lambda$. As shown in Figure~\ref{fig:simul_clients_inc}, PPFE consistently attains the lowest mean squared error (MSE) as the number of clients increases, while both Local and FedAvg baselines exhibit higher and more variable test losses.

\begin{figure}[!t]
    \centering
    \begin{subfigure}[t]{0.48\textwidth}
        \centering
        \includegraphics[width=\textwidth,trim={0 0 15mm 0},clip]{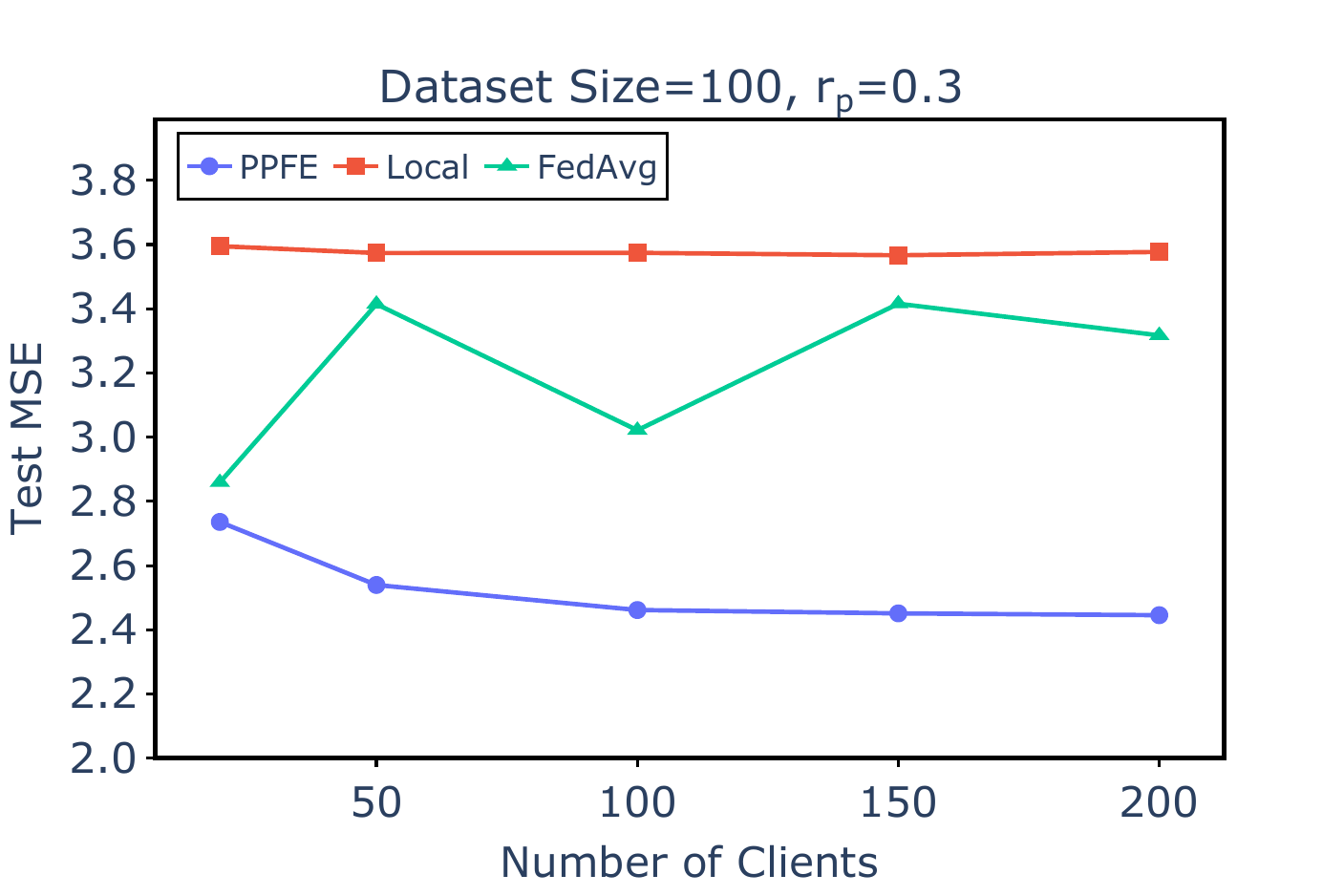}
        \caption{Comparison of test MSE across different numbers of clients.}
        \label{fig:simul_clients_inc}
    \end{subfigure}
    \hfill
    \begin{subfigure}[t]{0.48\textwidth}
        \centering
        \includegraphics[width=\textwidth,trim={0 0 15mm 0},clip]{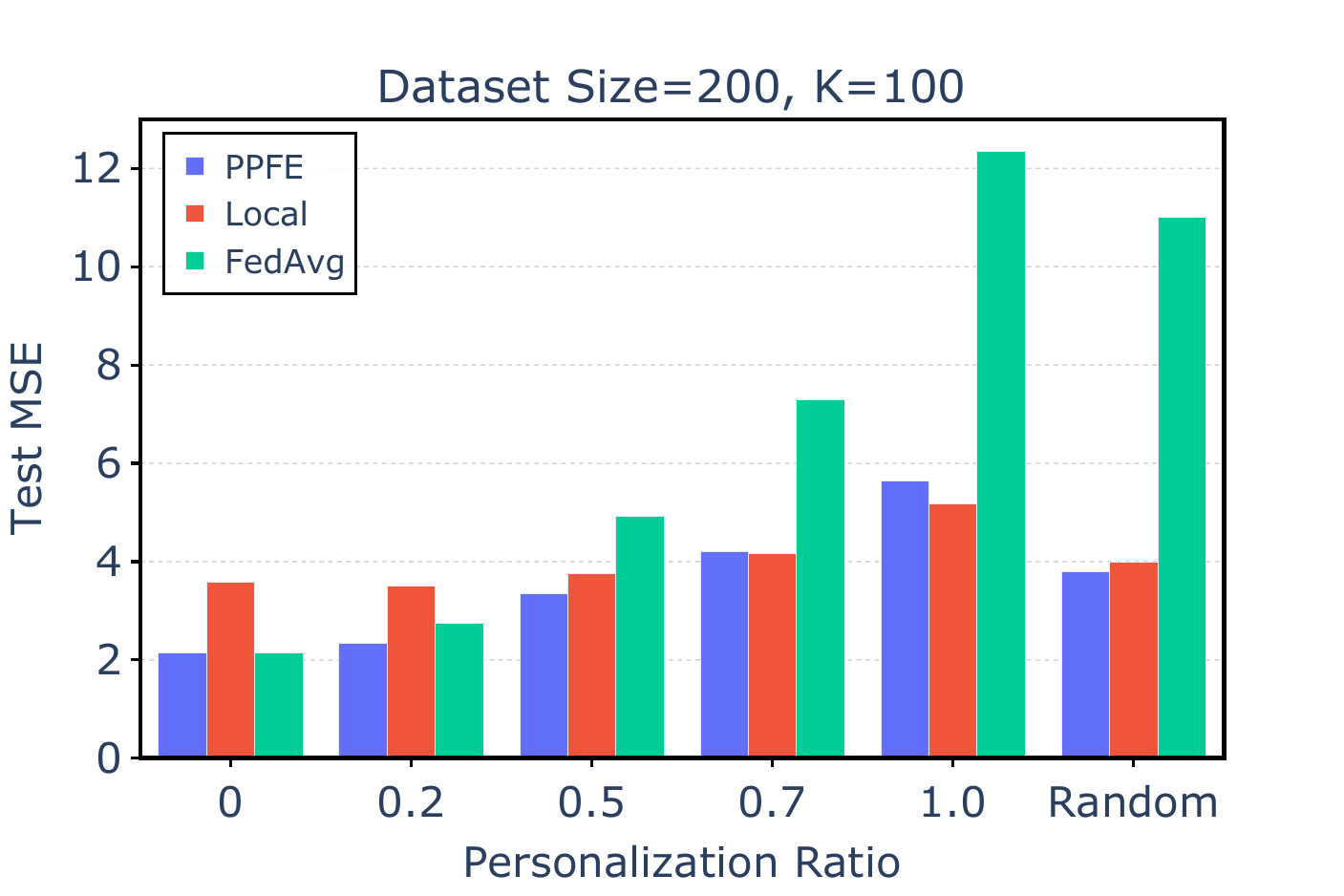}
        \caption{Effect of different personalization ratios.}
        \label{fig:simul_personal_ratio}
    \end{subfigure}
    \caption{Performance of PPFE on synthetic data under heterogeneous settings with varying numbers of clients and personalization ratios.}
    \label{fig:synthetic_results}
\end{figure}

Next, we examine the effect of the personalization ratio $r_p$ on model performance. We consider a setting with $K = 100$ clients and $n_k = 200$ samples per client, and vary $r_p$ from $0$ to $1$. The case $r_p = 0$ corresponds to identical underlying models across clients, while $r_p = 1$ represents a fully personalized regime with independent client-specific parameters. We also include a heterogeneous setting in which each client’s personalization ratio is independently sampled from the uniform distribution $\mathcal{U}(0,1)$. As shown in Figure \ref{fig:simul_personal_ratio}, the performance of FedAvg deteriorates rapidly as $r_p$ increases, whereas PPFE remains stable and consistently achieves low error across all levels of personalization.


\subsection{Real Data}
\label{sec:experiments:real}

To evaluate the performance of our algorithm in practical settings, we conduct experiments on real-world datasets using non-linear models under heterogeneous data distributions across clients.

\subsubsection{Models and Datasets}
We evaluate PPFE on four real-world datasets, including three image benchmarks (EMNIST Letters, CIFAR-10, and CIFAR-100) and one text benchmark (Sent140). For the image datasets, we employ convolutional neural network (CNN) models, while for Sent140 we use a recurrent neural network (RNN)-based architecture with pretrained word embeddings. Unless otherwise stated, we use $T=4$ stages in PPFE. To control the complexity of progressively personalized heads and reduce computational cost, we apply low-rank factorization to fully connected layers when constructing reduced models (details provided in \ref{app:models_datasets}).

\subsubsection{Data Partitioning}
To simulate heterogeneous data distributions across clients, we follow standard non-IID partitioning protocols commonly used in federated learning \cite{collins2021exploiting, makhija2022architecture}. For the image datasets, each client is restricted to a subset of classes, with the number of classes per client controlling the level of heterogeneity. For the Sent140 dataset, we adopt the natural partitioning by users. Detailed partitioning procedures are provided in \ref{app:data_partitioning}.

\subsubsection{Model Reduction}
To control the complexity of progressively personalized heads and mitigate overfitting, we employ model reduction mechanisms in PPFE. Our primary approach applies low-rank matrix factorization based on singular value decomposition (SVD) to the fully connected layers of the personalized head at each stage, reducing the effective dimensionality as additional layers become personalized. Using this strategy, we construct reduced variants of the base models that achieve different levels of parameter compression; for example, on CIFAR-10, the two reduced models correspond to approximately $6\%$ and $18\%$ reductions in the total number of parameters. This stage-wise reduction enables the capacity of the personalized head to increase gradually while keeping the overall model size manageable. In addition, we investigate an alternative reduction strategy based on structured sparsification of personalized parameters. Detailed reduction configurations and implementation specifics are provided in \ref{app:model_reduction}.

\subsubsection{Compared Methods}
We compare our method with baselines from personalized federated learning, global-model training, and non-cooperative local learning. Among personalized FL approaches, the closest to ours are model-decoupling methods, which split each client’s model into shared and client-specific parts but keep this partitioning fixed across communication rounds; this group includes FedBABU, FedRep, FedPer, and LG-FedAvg. We also include pFedFDA, which personalizes by adapting clients to a shared feature distribution. Another major line follows a multi-task learning formulation, where each client’s model is regularized to stay close to a global reference model, with Ditto being a strong representative. For global baselines, we evaluate FedAvg, the standard algorithm, and its personalized variant FedAvg+FT, which applies client-side fine-tuning in the final round. Finally, the Local-only baseline reflects the performance obtainable when each client trains independently on its own data.

\subsubsection{Implementation}
All experiments are conducted with 100 clients, where in each communication round a fraction $r=0.1$ of clients is randomly sampled to participate, except for the final round of each stage, which uses full client participation. Client-side optimization is performed using mini-batch SGD, with each selected client running 5 local epochs per communication round. PPFE is trained over $T=4$ stages with progressively increasing personalization. The initial stage uses FedAvg to learn a shared global model, which is subsequently reused for weight initialization and sample reweighting in later stages.

Importantly, the progressive training procedure does not introduce additional communication or computation overhead. The total number of communication rounds across all stages is fixed and matches that of all baseline methods. For image datasets, both PPFE and baseline algorithms are trained for 160 communication rounds, while for the Sent140 dataset all methods are trained for 50 rounds. This ensures a fair comparison under identical training budgets. Final performance is evaluated by averaging test accuracy across clients, using test data drawn from the same distribution as local training data. For PPFE, we report the accuracy of the resulting ensemble, whereas for baseline methods we report the accuracy of their final trained models. Detailed implementation settings are provided in \ref{app:implementation}.

\subsubsection{Performance Comparison}
We evaluate PPFE across multiple real-world datasets under varying degrees of data heterogeneity. For EMNIST, each client is assigned $S=10$ out of the 26 available classes with $n_k=100$ samples per client. For CIFAR-10 and CIFAR-100, we consider two heterogeneity regimes by varying the number of classes per client: high heterogeneity using $S=2$ for CIFAR-10 and $S=5$ for CIFAR-100, and low heterogeneity using $S=5$ and $S=20$, respectively. All clients are constructed with equally sized local datasets and a deliberately limited number of samples to reflect realistic federated learning conditions, with 150 samples per client for CIFAR-10 and 400 samples per client for CIFAR-100. Performance is evaluated using global test accuracy, computed as a weighted average of client accuracies according to local dataset sizes. 
Results across datasets and heterogeneity levels are reported in Table~\ref{tab:test_acc_comparision}, showing that PPFE attains the best performance in most configurations and offers consistent improvements over strong personalized baselines, particularly in highly heterogeneous scenarios. Results obtained using the alternative structured sparsification-based model reduction approach are reported in \ref{app:sparse_results}.

\begin{table}[!t]
\centering
\renewcommand{\arraystretch}{1.1}
\setlength{\tabcolsep}{5pt}
\scriptsize
\begin{threeparttable}
\begin{tabular}{l|cc|cc|c|c}
\toprule
 & \multicolumn{2}{c|}{CIFAR-10} & \multicolumn{2}{c|}{CIFAR-100} & EMNIST & Sent140 \\
\cline{2-7}
($K$, $S$) & (100, 2) & (100, 5) & (100, 20) & (100, 5) & (100, 10) & (100, 2) \\
\midrule
Local Only    & 81.66 & 59.36 & 36.05 & 72.61 & 72.40 & 71.13 \\
FedAvg        & 56.23 & 61.43 & 34.29 & 27.25 & 81.99 &  65.21 \\
FedAvg+FT     & 86.58 & 74.44 & 54.76 & 78.52 & 89.26 &  74.24 \\
Ditto         & 87.37 & 75.34 & 53.08 & 76.64 & \textbf{91.17} & 73.71 \\
FedBABU       & 86.45 & 76.10 & 46.14 & 73.31 & 84.50 & 72.99 \\
FedRep        & 82.16 & 69.29 & 40.81 & 73.52 & 77.42 & 74.46 \\
FedPer        & 87.29 & 74.74 & 43.72 & 75.90 & 83.07 & 75.34 \\
LG-FedAvg     & 86.55 & 61.08 & 39.30 & 74.08 & 76.12 & 73.64 \\
pFedFDA       & 88.55 & 73.82 & 50.59 & 74.83 & 90.86 & --\tnote{*} \\
\midrule
\textbf{PPFE} & \textbf{89.17} & \textbf{77.58} & \textbf{57.60} & \textbf{78.70} & 89.04 & \textbf{76.43} \\
\bottomrule
\end{tabular}
\begin{tablenotes}[flushleft]
\footnotesize
\item[*] ``--'' indicates that authors do not report results on text datasets.
\end{tablenotes}
\end{threeparttable}
\vspace{-3pt}
\caption{Test accuracy (\%) comparison across datasets under heterogeneous federated settings with a client participation ratio of $r=0.1$.}
\label{tab:test_acc_comparision}
\vspace{-5pt}
\end{table}

\subsubsection{Data Size and Data Heterogeneity Effect.}
We study the effect of local data size by varying it from 100 to 500 samples per client, using the CIFAR-100 dataset with $S = 5$ classes per client. As shown in Figure~\ref{fig:data-size}, federated learning methods consistently outperform local training, with PPFE achieving the highest accuracy. To evaluate the impact of data heterogeneity, we vary the number of classes per client from $S = 5$ to $S = 40$, keeping the local data size fixed. Figure~\ref{fig:data-hetero} shows that PPFE maintains superior performance across all heterogeneity levels, demonstrating robustness to diverse federated settings.

\begin{figure}[t]
  \centering
  \begin{subfigure}[b]{0.49\textwidth}
    \includegraphics[width=\linewidth]{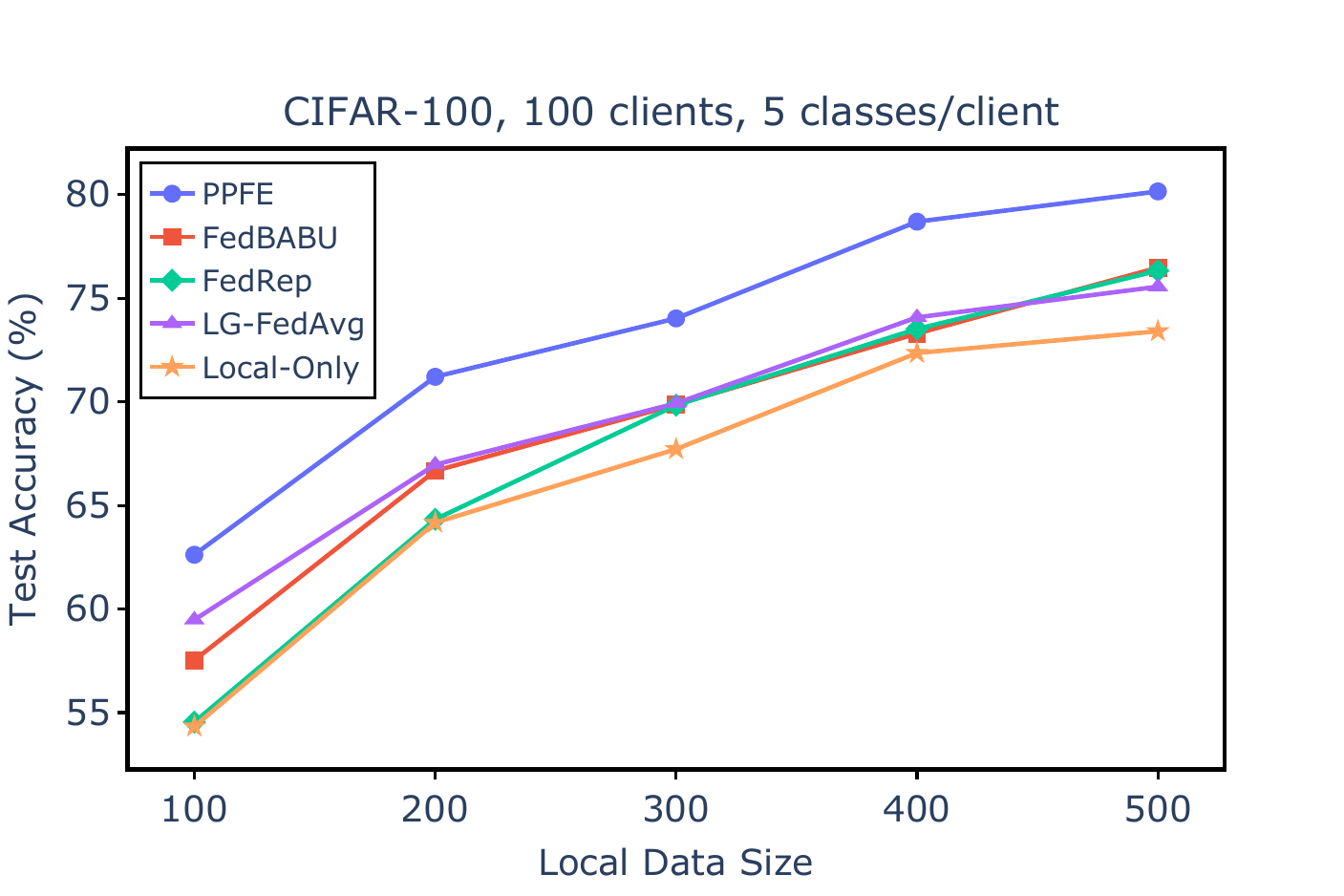}
    \caption{Performance over different local data sizes}
    \label{fig:data-size}
  \end{subfigure}
  \hfill
  \begin{subfigure}[b]{0.49\textwidth}
    \includegraphics[width=\linewidth]{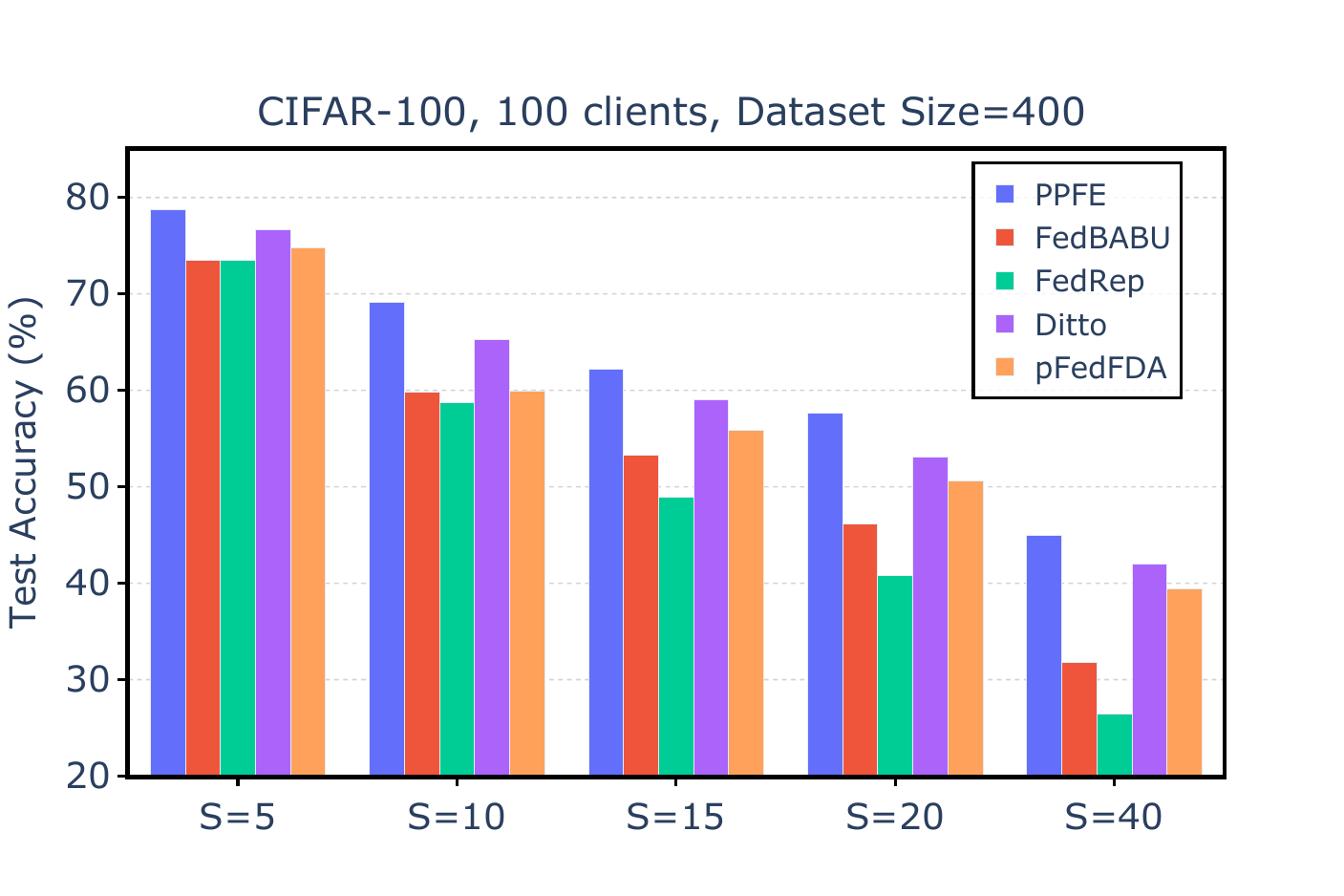}
    \caption{Performance over different data heterogeneity}
    \label{fig:data-hetero}
  \end{subfigure}
  \caption{Effect of data size and data heterogeneity on the performance.}
  \label{fig:data-settings}
\end{figure}

\subsubsection{Evaluation under Dirichlet-Based Data Heterogeneity}

Another widely used approach for constructing non-IID client datasets is the Dirichlet distribution, as adopted in \cite{lin2020ensemble}. The Dirichlet distribution is controlled by the hyperparameter $\alpha$, where larger values produce more identical local data distributions across clients and smaller values result in increasingly heterogeneous client data. In this experiment, we consider two heterogeneity levels by setting $\alpha = 1$ and $\alpha = 0.3$, evaluate our method on the CIFAR-10 dataset, and compare it with the baseline methods under the same experimental protocol as in our main results. The corresponding test accuracies are reported in Table~\ref{tab:dirichlet-allocation}, where our proposed method achieves the highest test accuracy in both Dirichlet settings.

\begin{table}[t]
    \centering
    \caption{Test accuracy (\%) comparison on CIFAR-10 under Dirichlet-based heterogeneous data distributions with a client participation ratio of $r=0.1$.}
    \label{tab:dirichlet-allocation}
    \resizebox{\linewidth}{!}{%
    \begin{tabular}{lcccccccccc}
        \toprule
        $\alpha$ & Local Only & FedAvg & FedAvg+FT & Ditto & FedBABU & FedRep & FedPer & LG-FedAvg & pFedFDA & PPFE \\
        \midrule
        1.0 & 45.20 & 60.82 & 65.95 & 59.49 & 64.47 & 57.71 & 62.88 & 48.05 & 65.84 & \textbf{67.17} \\
        0.3 & 61.18 & 65.91 & 75.24 & 68.77 & 74.47 & 69.46 & 73.65 & 63.51 & 73.72 & \textbf{76.53} \\
        \bottomrule
    \end{tabular}%
    }
\end{table}

\subsection{Ablation Studies}
We conduct a series of ablation studies to analyze the contribution of key components in PPFE and to justify the main design choices of the proposed framework.

\subsubsection{Model Reduction Effect.}
A key component of our method is the progressive reduction of complexity in the personal head across successive phases. This mechanism helps prevent overfitting as additional head layers and parameters are introduced. To evaluate its contribution, we performed an ablation study in which model reduction was disabled, and new layers were added to the personal head without any complexity control. The results in Table~\ref{tab:reduction_comparision} demonstrate that applying model reduction leads to improved final accuracy.

\begin{table}[!t]
\centering
\renewcommand{\arraystretch}{1.2}
\small
\begin{tabular}{@{}p{5cm}cc@{}}
\toprule
\textbf{Dataset (Setting)} & \textbf{With MR} & \textbf{Without MR} \\
\midrule
CIFAR-10 \ (n=100, S=2)     & 89.17 & 86.44 \\
CIFAR-10 \ (n=100, S=5)     & 77.58 & 76.10 \\
CIFAR-100 \ (n=100, S=5)    & 78.70 & 76.94 \\
CIFAR-100 \ (n=100, S=20)   & 57.60 & 54.07 \\
EMNIST \ (n=100, S=10)      & 89.04 & 87.73 \\
Sent140 \ (n=100, S=2) & 76.43 & 73.76 \\
\bottomrule
\end{tabular}
\caption{Effect of model reduction (MR) on final accuracy across datasets.}
\label{tab:reduction_comparision}
\end{table}

\subsubsection{Comparison with Fixed Personalized Heads}
To evaluate the effectiveness of progressive personalization, we compare PPFE with personalized federated learning models that employ fixed head architectures of varying depth. We construct three baselines, denoted as M1, M2, and M3, in which the last one, two, and three layers of the network are personalized, respectively, while the remaining layers are shared across clients. All fixed-head models are initialized from a common global model trained using FedAvg and are subsequently trained independently under the same communication budget to ensure a fair comparison. As shown in Table~\ref{tab:diff_personal_ratio}, increasing the number of personalized layers in fixed-head models generally leads to degraded performance, indicating overfitting under limited local data and heterogeneous client distributions. Among these baselines, the model with a single personalized layer (M1) achieves the strongest results, yet it remains consistently inferior to PPFE across all datasets. In contrast, PPFE achieves higher accuracy by progressively increasing the degree of personalization while controlling model complexity, demonstrating the advantage of gradual personalization over static head designs in personalized federated learning.

\begin{table}[!tb]
\centering
\renewcommand{\arraystretch}{1.2}
\small
\begin{tabular}{@{}p{5cm}cccc@{}}
\toprule
\textbf{Dataset (Setting)} & \textbf{M1} & \textbf{M2} & \textbf{M3} & \textbf{PPFE} \\
\midrule
CIFAR-10 \ (n=100, S=2)    & 88.28 & 87.36 & 86.59 & 89.17 \\
CIFAR-10 \ (n=100, S=5)    & 76.41 & 73.94 & 67.97 & 77.58 \\
CIFAR-100 \ (n=100, S=5)   & 75.86 & 75.11 & 74.36 & 78.70 \\
CIFAR-100 \ (n=100, S=20)  & 51.58 & 47.20 & 45.01 & 57.60 \\
EMNIST \ (n=100, S=10)     & 87.94 & 86.74 & 84.89 & 89.04 \\
Sent140 \ (n=100, S=2)     & 73.93 & 73.55 & 73.46 & 76.43 \\
\bottomrule
\end{tabular}
\caption{Accuracy of personalized FL models with different head depths. \textbf{M1}, \textbf{M2}, and \textbf{M3} correspond to models with one, two, and three personalized layers, respectively.}
\label{tab:diff_personal_ratio}
\end{table}

\subsubsection{Effect of Progressive Personalization}
A central component of PPFE is its progressive personalization strategy, in which the depth of the personalized head is gradually increased across training stages to better capture client-level data heterogeneity. This design allows the model to balance global representation learning and local adaptation while also reducing communication cost, as fewer parameters are exchanged when a larger portion of the network remains client-specific. For example, on CIFAR-100, clients transmit only $98.8\%$, $83.8\%$, and $19.5\%$ of the original model parameters in stages $T_1$, $T_2$, and $T_3$, respectively. To assess the impact of this strategy, we conduct an ablation study in which progressive personalization is removed by fixing the personalized head to a single layer after the initial FedAvg stage and continuing training using FedPer. We evaluate this setting both with and without inter-stage sample reweighting. As shown in Table~\ref{tab:ablation_progression}, removing progressive personalization consistently degrades performance across all datasets, while additionally disabling reweighting further amplifies this drop. These results confirm that both gradual head expansion and reweighting play complementary roles in improving accuracy, validating the effectiveness of the progressive personalization mechanism in PPFE.

\begin{table}[!t]
\centering
\renewcommand{\arraystretch}{1.2}
\small
\begin{tabular}{@{}p{5cm}cccc@{}}
\toprule
\textbf{Dataset (Setting)} & \textbf{WP} & \textbf{WPW} & \textbf{PPFE} \\
\midrule
CIFAR-10 \ (n=100, S=2)     & 88.07 & 88.12 & 89.17  \\
CIFAR-10 \ (n=100, S=5)     & 76.61 & 75.75 & 77.58 \\
CIFAR-100 \ (n=100, S=5)    & 76.97 & 77.21 & 78.70 \\
CIFAR-100 \ (n=100, S=20)   & 50.69 & 51.39 & 57.60 \\
EMNIST \ (n=100, S=10)      & 88.08 & 87.88 & 89.04 \\
Sent140 \ (n=100, S=2)      & 73.58 & 74.35 & 76.43 \\
\bottomrule
\end{tabular}
\caption{Ablation study on head model progression across training phases. \textbf{WP} denotes the removal of progressive personalization, and \textbf{WPW} indicates that sample reweighting between phases is also removed.}
\label{tab:ablation_progression}
\end{table}

\section{Conclusion}
\label{sec:conclusion}

In this paper, we proposed PPFE, a progressive ensemble-based personalized federated learning framework that balances global representation learning with client-specific adaptation in non-IID settings. While naively increasing the number of personalized layers can lead to overfitting and degraded generalization, PPFE addresses this challenge by incrementally increasing local model capacity under explicit complexity control, followed by a boosting procedure with inter-stage sample reweighting to emphasize residual errors. PPFE consistently outperforms strong decoupling baselines such as FedRep \cite{collins2021exploiting} and FedBABU \cite{Oh2021FedBABUTE} across both synthetic and real-world benchmarks, demonstrating robustness in practical heterogeneous federated environments. The method improves personalization performance while remaining fully compatible with standard federated optimization pipelines and model architectures, making it well suited for privacy-preserving collaborative learning in domains such as healthcare, finance, and mobile or IoT analytics.

We further provide theoretical generalization guarantees that support the observed empirical gains under heterogeneous data distributions. In particular, our analysis yields explicit characterizations of the bias--variance trade-off induced by progressively assigning personalized layers across stages $t \in [T]$ (e.g., Corollary \ref{corl:specialSetting:informal}). A promising direction for future work is to move beyond the current shared reduction schedule across clients and develop client-adaptive strategies that better account for local data characteristics and resource constraints, with the potential to further improve both efficiency and generalization.

\bibliographystyle{elsarticle-num}
\bibliography{ref}

\appendix
\section{Details of Experimental Setup}
\label{sec:app:experiments}
All experimental evaluations were conducted using the PyTorch framework and executed on NVIDIA Tesla P100 GPUs. The implementation details and code used to produce the reported results are included in the supplementary materials and also is available at \url{https://anonymous.4open.science/r/PPFE-6E7C}.

\subsection{Models and Datasets}
\label{app:models_datasets}

We evaluate our method on four real-world datasets covering both vision and text modalities. For image classification, we consider EMNIST Letters, CIFAR-10, and CIFAR-100. EMNIST Letters is a handwritten character recognition dataset derived from the Extended MNIST collection, consisting of grayscale images of handwritten alphabet characters spanning 26 classes. CIFAR-10 is a widely used benchmark composed of 60000 color images evenly distributed across 10 object categories, while CIFAR-100 is a more challenging variant with the same image resolution but 100 fine-grained classes, introducing a higher degree of semantic complexity. For text-based evaluation, we use the Sent140 dataset, which contains short social media messages annotated for binary sentiment and naturally partitioned by user, making it well suited for federated learning scenarios. For the image datasets, we employ convolutional neural network architectures composed of two convolutional layers (with dataset-dependent channel widths) followed by max-pooling and three fully connected layers, where the convolutional layers act as a shared feature extractor and the fully connected layers form the personalized head that is progressively expanded in PPFE. For Sent140, text inputs are represented using 300-dimensional pretrained GloVe embeddings and processed by a recurrent neural network, followed by three fully connected layers, with personalization applied in the same progressive manner. All clients share the same base architecture within each dataset to ensure fair and consistent evaluation.

\subsection{Data Partitioning}
\label{app:data_partitioning}

To construct heterogeneous client datasets, we follow a class-restriction strategy commonly adopted in personalized federated learning \cite{collins2021exploiting, makhija2022architecture}. For the image datasets, each client is assigned data from a limited subset of classes, while different clients observe different class subsets. The level of heterogeneity is controlled by the number of classes $S$ allocated to each client. For example, in the CIFAR-10 dataset with $S=5$, one client may have access to samples from classes $\{3,7,1,9,4\}$, while another client may observe classes $\{6,2,8,3,7\}$. This setup induces statistical heterogeneity while maintaining balanced local dataset sizes across clients. For the Sent140 dataset, we adopt the natural partitioning induced by tweet authorship, where each client corresponds to a user. To ensure reliable evaluation, users with fewer than eight test samples are removed, resulting in 136 clients with an average of approximately 99 training samples per client. This partitioning reflects realistic federated learning scenarios with diverse and uneven local data distributions.

\subsection{Model Reduction}
\label{app:model_reduction}

We primarily reduce the complexity of personalized components using low-rank matrix factorization applied to the fully connected layers of the personalized head. Specifically, at each PPFE stage, the newly personalized layers are replaced with lower-rank approximations obtained via singular value decomposition (SVD), resulting in reduced variants of the base models with controlled capacity. For the CIFAR-10 dataset, the first and second reduced models use factorization ranks of $(100, 50, 24)$ and $(80, 40, 16)$ for the three fully connected layers, respectively, corresponding to approximately $6\%$ and $18\%$ reductions in the total number of trainable parameters. For the remaining datasets, factorization ranks are selected to achieve comparable relative reductions in model size. In all cases, convolutional layers in image models and recurrent layers in the text model are kept unchanged to preserve feature extraction capacity.

As a parallel alternative to low-rank factorization, we also explore structured sparsification of personalized parameters using fixed binary masks shared across clients. In this approach, $30\%$ of the parameters in the personalized fully connected layers are masked during the initial personalization stage, and when additional layers become personalized in later stages, an extra $10\%$ of parameters in the newly introduced layer are masked. Identical sparsity patterns are enforced across clients through a fixed random seed, ensuring that all clients optimize within the same sparse subspace and enabling meaningful aggregation of personalized parameters.

\subsection{Implementation Details}
\label{app:implementation}

All experiments involve 100 clients, with a participation ratio of $r=0.1$ per communication round, except for the final round of each stage, which uses full client participation. For the image datasets, the initial stage ($T_0$) is trained using the FedAvg algorithm for 30 epochs to obtain a global model, which is then used for weight initialization and client-side sample reweighting in subsequent stages. In the second stage, clients jointly train a shared feature extractor together with a personalized single-layer head for 50 communication rounds. Two reduced models are obtained by applying SVD-based low-rank factorization to the fully connected layers of the FedAvg model.

In the third stage, models with two personalized head layers are trained for 40 rounds, followed by another sample reweighting step, after which convolutional weights are transferred to the final stage. In the fourth stage, only the convolutional layers remain global, while all remaining layers are trained locally. Client-side optimization uses mini-batch SGD with momentum, where each selected client performs 5 local epochs per round. The batch size is set to 10 for image datasets and 2 for Sent140. The initial learning rate is $\eta=0.01$ and is reduced to $0.001$ for the shared body parameters in the later stages.

For the Sent140 dataset, the four stages are trained for 10, 20, 10, and 10 rounds, respectively.

\section{Additional Experiments}
\label{app:sparse_results}

\begin{table}[t]
\centering
\small
\begin{tabular}{l|cc|cc|c|c}
\toprule
 & \multicolumn{2}{c|}{CIFAR-10} & \multicolumn{2}{c|}{CIFAR-100} & EMNIST & Sent140 \\
\cline{2-7}
($K$, $S$) & (100, 2) & (100, 5) & (100, 20) & (100, 5) & (100, 10) & (100, 2) \\
\midrule
PPFE (SVD)    & 89.17 & 77.58 & 57.60 & 78.70 & 89.04 & 76.43 \\
PPFE (Sparse) & 89.14 & 77.21 & 54.17 & 76.83 & 89.45 & 75.94 \\
\bottomrule
\end{tabular}
\caption{Comparison of PPFE using low-rank (SVD) and sparsification-based model reduction under heterogeneous data distributions.}
\label{tab:sparse_results}
\end{table}

In addition to the primary low-rank factorization strategy used in PPFE, we conduct supplementary experiments using a structured sparsification-based approach to control the complexity of personalized heads. In this setting, personalized parameters are constrained using fixed binary masks shared across clients, enforcing a common sparse subspace during training. This alternative reduction mechanism is applied following the same progressive personalization schedule and under the same communication and computation budgets as the main experiments.

Table~\ref{tab:sparse_results} reports the test accuracy obtained with sparsification-based model reduction across the evaluated datasets and heterogeneity settings. While this approach generally yields slightly lower accuracy compared to the SVD-based reduction strategy used in the main paper, it consistently outperforms or remains competitive with strong personalized federated learning baselines. These results indicate that PPFE is robust to different forms of model complexity control and that its performance gains are not tied to a specific reduction technique.

\section{Theoretical Analysis: Complete Theorems and Proofs}
\label{sec:app:theory}



In this section, we present a full theoretical analysis of the proposed PPFE method introduced in Section \ref{sec:proposed}, with particular focus on Algorithm \ref{alg:ppfe}. The reader will find the complete statements of the theorems and corollaries referenced in Section \ref{sec:theory}.

The structure of this section is as follows. Section \ref{sec:theory:prelim} introduces the notation, definitions, and fundamental lemmas from statistical learning theory that will be used throughout the analysis. Section \ref{sec:theory:specialtools} specializes these general tools to the case of feed-forward \emph{deep neural networks}. In particular, we derive bounds on the covering numbers of a broad class of neural networks and analyze how coverings behave under compositions of the global and local (personalized) components of the overall DNN architecture. Finally, Section \ref{sec:app:theory:main} provides the complete statements of our main results (the generalization bounds), together with their full proofs and the resulting corollaries.


\subsection{Preliminaries}
\label{sec:theory:prelim}

Consider a general classification or regression task with feature vectors in a space $\mathcal{X}$ and labels or responses in $\mathbb{R}$. Let $\mathcal{X}\subseteq \mathbb{R}^d$ be a measurable feature space, where $d\in\mathbb{N}$ denotes the ambient dimension. In a parametric setting, let $\mathcal{F}\subseteq \mathbb{R}^{\mathcal{X}}$ be a family of functions mapping $\mathcal{X}$ to $\mathbb{R}$; that is, for every $f\in\mathcal{F}$, we have $f:\mathcal{X}\to\mathbb{R}$. We typically aim to minimize the expected loss of $f$ under a given loss function $\ell:\mathbb{R}\times\mathbb{R}\to\mathbb{R}$ and a joint feature–label distribution $P$ over $\mathcal{X}\times\mathbb{R}$:
$$
\inf_{f\in\mathcal{F}}~\mathbb{E}_P\!\left[
\ell\bigl(y,f(\mathbf{X})\bigr)
\right],
$$
where $(\mathbf{X},y)\sim P$. In this context, we define the $\varepsilon$-cover of $\mathcal{F}$ under the $\ell_{\infty}$ norm as follows.

\begin{definition}[$\varepsilon$-Cover and Covering Number]
For a function class $\mathcal{F}$ and $\varepsilon\ge 0$, the $\varepsilon$-cover 
$\mathsf{Cov}\!\left(\mathcal{F},\varepsilon,\|\cdot\|_{\infty}\right)$ is defined as
\begin{align}
\mathsf{Cov}\!\left(\mathcal{F},\varepsilon,\|\cdot\|_{\infty}\right)
\triangleq
\arg\min_{S\subset \mathcal{F}}
\left|
\left\{
S\,\middle|\,
\forall f\in\mathcal{F},~\exists\,\widehat{f}\in S
\text{ such that }
\|f-\widehat{f}\|_{\infty}\le \varepsilon
\right\}
\right|.
\end{align}
In words, $\mathsf{Cov}\!\left(\mathcal{F},\varepsilon,\|\cdot\|_{\infty}\right)$ is the smallest finite subset of $\mathcal{F}$ that approximates every function in $\mathcal{F}$ within $\varepsilon$ in the $\ell_{\infty}$ norm. The \emph{covering number} of $\mathcal{F}$ is then defined as
$$
\mathcal{N}\!\left(\mathcal{F},\varepsilon,\|\cdot\|_{\infty}\right)
\triangleq
\left|
\mathsf{Cov}\!\left(\mathcal{F},\varepsilon,\|\cdot\|_{\infty}\right)
\right|.
$$
\end{definition}

\begin{lemma}[From {\cite{mohri2018foundations}}]
\label{lemm:RadUniformBound}
Let $\mathcal{D}=\{\mathbf{X}_1,\ldots,\mathbf{X}_n\}$ consist of $n\ge 1$ i.i.d.\ samples from a fixed distribution $P$ supported on $\mathcal{X}$. Then, for any $P$-measurable function class $\mathcal{F}$ and any $\zeta\in(0,1)$, the following uniform convergence bound holds:
\begin{align}
\mathbb{P}\Biggl(
\sup_{f\in\mathcal{F}}
\biggl|
\mathbb{E}_P\!\left[f(\mathbf{X})\right]
-
\frac{1}{n}\sum_{i=1}^{n} f(\mathbf{X}_i)
\biggr|
\le
\inf_{\varepsilon\ge 0}
\Biggl\{
\varepsilon
+
\sqrt{
\frac{
\log \mathcal{N}\!\left(
\mathcal{F},\varepsilon,\|\cdot\|_{\infty}
\right)
}{n}
}
\Biggr\}
+
\sqrt{\frac{\log (2/\zeta)}{2n}}
\Biggr)
\ge 1-\zeta .
\nonumber
\end{align}
where the probability is taken over the random draw of the sample $\mathcal{D}$.
\end{lemma}


\subsection{Special Tools for Deep Neural Networks}
\label{sec:theory:specialtools}

Assume that we work with deep neural network architectures. The core of our analysis is general, and the reader may extend it to a broad range of hypothesis classes. For instance, convolutional neural networks or DNNs with sparsity constraints can be treated by straightforward adaptations of the arguments presented here. We emphasize that our analysis relies on covering-number bounds, which are typically vacuous for deep models in low-sample regimes. Obtaining non-vacuous and fully general bounds for deep neural networks remains largely an open problem. Nevertheless, our goal is to characterize how the generalization guarantees scale with parameters such as $K$, $n_{1:K}$, the depth $L$, and the number of boosting iterations $T$.

Consider a feedforward fully connected neural network of depth $L$. Let $N_\ell$, for $\ell=0,1,\ldots,L-1$, denote the width of the network at the $\ell$-th layer (for a single-output network we set $N_L=1$). The input dimension satisfies $N_0=d$. For $\ell=1,2,\ldots,L-1$, the widths $N_{\ell}$ may be arbitrary. The weight matrix connecting layer $\ell$ to $\ell+1$ therefore has dimensions $N_{\ell}\times N_{\ell+1}$ in addition to a bias vector of dimension $N_{\ell+1}$. We denote the activation function at each layer by $\sigma:\mathbb{R}\to\mathbb{R}$, which we assume to be increasing and Lipschitz.

Recall that our method uses a boosting strategy with $T$ iterations. At iteration $t\in[T]$, we increase the depth of the personalized portion of the network by one layer, starting from zero (corresponding to the vanilla FedAvg method). Let $\Phi^{(t)}$ and $\Psi^{(t)}$, for $t=1,2,\ldots,T$, denote the weight spaces of the personalized and common components of the architecture at stage $t$. Thus,
$$
\Phi^{(t)}\subseteq \mathbb{R}^{D_t}
\quad,
\quad
\Psi^{(t)}\subseteq \mathbb{R}^{D'_t},
$$
where parameter sizes are defined as:
\begin{align}
D_t \triangleq \sum_{l=L-t+1}^{L} \left(N^{(t)}_{l}N^{(t)}_{l-1} + N^{(t)}_l\right)
~~ (t>1)\quad,\quad
D'_t \triangleq \sum_{l=1}^{L-t} \left(N^{(t)}_{l}N^{(t)}_{l-1} + N^{(t)}_l\right).
\end{align}
Note that we define $D_1=0$, since there are no personalized layers at iteration $t=1$. We use $\subseteq$ rather than equality to allow for possible regularization (e.g., sparsity constraints or $\ell_p$-norm constraints) that may restrict the feasible weight space. However, in the final bounds we assume no such regularization is enforced. However, similar to almost other works in the literature we assume all the weights of the neural networks are bounded by a sufficiently large positive constant $B\ge0$, i.e., the cannot reach $\infty$ during the training. Otherwise, this class does not admit a bounded complexity measure.

For each $\phi\in\Phi^{(t)}$, let $g_{\phi}$ denote the neural network \emph{function} corresponding to the weights $\phi$. Similarly, for each $\psi\in\Psi^{(t)}$, let $f_{\psi}$ denote the network function associated with $\psi$, mapping the input $\mathbf{X}$ to its representation after the $t$-th layer. Formally,
$$
f_{\psi}:\mathbb{R}^d\to \mathbb{R}^{N_{L-t}}
\quad,\quad
g_{\phi}:\mathbb{R}^{N_{L-t}}\to \mathbb{R}.
$$
In the following subsection, we derive covering number bounds for the classes of functions $\{f_{\psi}\}$ and $\{g_{\phi}\}$, as well as for their compositions $g_{\phi}\circ f_{\psi}$.

\begin{lemma}[Covering Number of Fully-Connected DNNs] 
\label{lemma:DNNCoveringBound}
Let $N_0,N_1,\dots,N_L$ be positive integers with $N_L=1$. Fix an activation function $\sigma:\mathbb R\to\mathbb R$ which is $\theta$--Lipschitz, i.e.
$
|\sigma(u)-\sigma(v)|\le \theta|u-v|~\text{for all }u,v\in\mathbb R,
$
and assume $|\sigma(0)|\le S_0$ for some $S_0\ge0$. Let the input domain be
$
\mathcal X=\{\boldX\in\mathbb{R}^{N_0}:\|\boldX\|_\infty\le R\},
$
for some $R\ge0$. Consider the class $\mathcal F$ of fully-connected feed-forward networks with $L$ layers where, for each layer $l=1,\dots,L$, the weight matrix is $W^{(l)}\in\mathbb R^{N_l\times N_{l-1}}$ and bias vector $b^{(l)}\in\mathbb R^{N_l}$, and the layer map is
$$
a^{(l)}=W^{(l)} z^{(l-1)}+b^{(l)}~,~ z^{(l)}=\sigma\big(a^{(l)}\big),
$$
with $z^{(0)}=\boldX$ and $z^{(L)}$ the scalar network output. Assume every entry of every parameter (all weights and all bias components) is bounded in absolute value by some $B>0$.
Then for every $\varepsilon>0$ the covering number of $\mathcal F$ with respect to the sup-norm on $\mathcal X$ satisfies the explicit bound
$$
\log\mathcal{N}\left(\mathcal{F},\varepsilon,\|\cdot\|_{\infty}\right)
\leq
\left[
\sum_{l=1}^L\big(N_lN_{l-1}+N_l\big)
\right]
\left(
\log\left(\tfrac{1}{\varepsilon}\right)+\widetilde{\mathcal{O}}(L)
\right),
$$
where $\widetilde{\mathcal{O}}(\cdot)$ hides all logarithmic dependence on $B,R,\theta$ and $S_0$.
\end{lemma}
\begin{proof}
The proof proceeds in three steps: i) Bound how parameter perturbations (in entrywise $\ell_\infty$) affect layer pre-activations and activations; ii) Iterate this bound to obtain a Lipschitz-in-parameters constant $C_{\mathrm{net}}$ for the map $\theta\mapsto f_\theta$ on the domain $\mathcal X$;
iii) The total number of scalar parameters is
$D\triangleq\sum_{l=1}^L\big(N_lN_{l-1}+N_l\big)$. Discretize the parameter hypercube $[-B,B]^D$ with mesh size $\delta=\varepsilon/C_{\mathrm{net}}$ and count grid points.
\\[1mm]
\noindent
\textbf{Stage i)}
Fix two networks with parameter collections
$$
\Theta=\{W^{(l)},b^{(l)}\}_{l=1}^L,\qquad \widetilde\Theta=\{\widetilde W^{(l)},\widetilde b^{(l)}\}_{l=1}^L,
$$
where every scalar parameter (every entry of every $W^{(l)}$ and every coordinate of every $b^{(l)}$) lies in $[-B,B]$. Denote by $z^{(\ell)}$ and $\widetilde z^{(\ell)}$ the activations produced by $\Theta$ and $\widetilde\Theta$, respectively, on the same input $x\in\mathcal X$. Set
$$
\Delta_\ell\triangleq\|z^{(\ell)}-\widetilde z^{(\ell)}\|_\infty\quad(\ell=0,\dots,L),
$$
so $\Delta_0=\|x-x\|_\infty=0$, and the final output discrepancy is $\Delta_L$ (the scalar absolute difference). For a matrix $A\in\mathbb R^{m\times n}$ denote by $\|A\|_{\infty,\mathrm{row}}$ the usual matrix infinity norm (maximum absolute row sum). If every entry of $A$ is bounded by $u$ in absolute value, then
$$
\|A\|_{\infty,\mathrm{row}}\le n\,u,
$$
since each row sum has at most $n$ terms each of magnitude $\le u$. For layer $\ell$ we have
$$
a^{(\ell)}- \widetilde a^{(\ell)}=(W^{(\ell)}-\widetilde W^{(\ell)})\,z^{(\ell-1)}+\widetilde W^{(\ell)}\,(z^{(\ell-1)}-\widetilde z^{(\ell-1)})+(b^{(\ell)}-\widetilde b^{(\ell)}).
$$
Taking $\ell_\infty$ norms and using submultiplicativity with the matrix row-norm,
$$
\|a^{(\ell)}- \widetilde a^{(\ell)}\|_\infty
\le \|W^{(\ell)}-\widetilde W^{(\ell)}\|_{\infty,\mathrm{row}}\,\|z^{(\ell-1)}\|_\infty
+\|\widetilde W^{(\ell)}\|_{\infty,\mathrm{row}}\,\|z^{(\ell-1)}-\widetilde z^{(\ell-1)}\|_\infty
+\|b^{(\ell)}-\widetilde b^{(\ell)}\|_\infty.
$$
If every entrywise parameter perturbation is at most $\delta$, i.e.
$$
\max\{\|W^{(\ell)}-\widetilde W^{(\ell)}\|_\infty,\|b^{(\ell)}-\widetilde b^{(\ell)}\|_\infty\}\le\delta\quad\text{for all }\ell,
$$
then, since each $W^{(\ell)}-\widetilde W^{(\ell)}$ has $N_{\ell-1}$ columns and each entry bounded by $\delta$,
$
\|W^{(\ell)}-\widetilde W^{(\ell)}\|_{\infty,\mathrm{row}}\le N_{\ell-1}\,\delta.$
Also, because entries of each $\widetilde W^{(\ell)}$ are bounded by $B$,
$
\|\widetilde W^{(\ell)}\|_{\infty,\mathrm{row}}\le N_{\ell-1}\,B.
$
Thus
$$
\|a^{(\ell)}- \widetilde a^{(\ell)}\|_\infty
\le N_{\ell-1}\delta\,\|z^{(\ell-1)}\|_\infty + N_{\ell-1}B\,\Delta_{\ell-1} + \delta.
$$
Applying the $\theta$--Lipschitz property of $\sigma$ componentwise,
$$
\Delta_\ell=\|z^{(\ell)}-\widetilde z^{(\ell)}\|_\infty
=\|\sigma(a^{(\ell)})-\sigma(\widetilde a^{(\ell)})\|_\infty
\le \theta\,\|a^{(\ell)}- \widetilde a^{(\ell)}\|_\infty.
$$
Combine the last two displays to obtain the per-layer recursion
$$
\Delta_\ell \le \theta\big(N_{\ell-1}\delta\,\|z^{(\ell-1)}\|_\infty + N_{\ell-1}B\,\Delta_{\ell-1} + \delta\big).
$$
\\[1mm]
\noindent
\textbf{Stage ii)}
Define the deterministic sequences $\{M_\ell\}_{\ell=0}^L$ by
$$
M_0\triangleq R,\quad M_\ell\triangleq\theta\big(N_{\ell-1} B\,M_{\ell-1}+B\big)+S_0\quad(\ell=1,\dots,L).
$$
Further, for each $t\in\{1,\dots,L\}$ define the product
$
P_{t}\triangleq\prod_{j=t+1}^{L}\big(\theta\,N_{j-1}B\big),
$
with the convention that an empty product equals $1$ (so $P_L=1$).
Set
$$
C_{\mathrm{net}}\triangleq\theta\sum_{t=1}^L P_{t}\,\big(N_{t-1}M_{t-1}+1\big).
$$
From
$
a^{(\ell)}=W^{(\ell)} z^{(\ell-1)}+b^{(\ell)}
$
we get
$$
\|a^{(\ell)}\|_\infty \le \|W^{(\ell)}\|_{\infty,\mathrm{row}}\,\|z^{(\ell-1)}\|_\infty + \|b^{(\ell)}\|_\infty
\le N_{\ell-1}B\,\|z^{(\ell-1)}\|_\infty + B.
$$
Thus
$$
\|z^{(\ell)}\|_\infty=\|\sigma(a^{(\ell)})\|_\infty
\le \theta\|a^{(\ell)}\|_\infty + S_0
\le \theta\big(N_{\ell-1}B\,\|z^{(\ell-1)}\|_\infty + B\big)+S_0.
$$
Iterating this with $\|z^{(0)}\|_\infty\le R$ produces the sequence $M_\ell$, and therefore
$
\|z^{(\ell)}\|_\infty\le M_\ell~(\ell=0,\dots,L).
$
Insert this bound into the recursion for $\Delta_\ell$ to obtain
$$
\Delta_\ell \le \theta N_{\ell-1}B\,\Delta_{\ell-1} + \theta\delta\big(N_{\ell-1}M_{\ell-1}+1\big).
$$
Unrolling this affine recursion with $\Delta_0=0$ yields
$$
\Delta_L
\le \theta\delta\sum_{t=1}^{L}\Big(\prod_{j=t+1}^{L}\big(\theta N_{j-1}B\big)\Big)\,\big(N_{t-1}M_{t-1}+1\big)
= \delta\,C_{\mathrm{net}}.
$$
Thus for all $\boldX\in\mathcal X$,
$
|f_\Theta(\boldX)-f_{\widetilde\Theta}(\boldX)|\le \delta\,C_{\mathrm{net}}.
$
\\[1mm]
\noindent
\textbf{Stage iii)}
The parameter vector lies in $[-B,B]^D$. A uniform grid of mesh
$
\delta=\frac{\varepsilon}{C_{\mathrm{net}}}
$
forms a $\delta$-cover of this hypercube in entrywise $\ell_\infty$ norm with at most $(2B/\delta)^D$ points. By the Lipschitz property proved above, each such grid point induces an $\varepsilon$-ball in function space (sup-norm over $\mathcal X$). Therefore
$$
\mathcal{N}\left(\mathcal{F},\varepsilon,\|\cdot\|_{\infty}\right)
\le\Big(\frac{2B}{\delta}\Big)^{D}
=\Big(\frac{2B\,C_{\mathrm{net}}}{\varepsilon}\Big)^{D}.
$$
This proves the lemma.
\end{proof}

\begin{lemma}[Compositional Coverage]
\label{lemma:compositionCovering}
Assume function family $\mathcal{F}$ mapping $\mathcal{X}$ to $\mathbb{R}^l$, and function family $\mathcal{G}$ mapping $\mathbb{R}^l$ to $\mathbb{R}$. Also, assume each function $f\in\mathcal{F}$ is $\gamma$-Lipschitz in $\ell_{\infty}$-norm for some fixed $\gamma\ge0$. Let us define $\mathcal{F}\circ\mathcal{G}\triangleq\left\{f\circ g\vert~ f\in\mathcal{F},~g\in\mathcal{G}\right\}$. Then, for any $\varepsilon_1,\varepsilon_2\ge0$, the finite set
$$
\mathsf{Cov}\left(\mathcal{F},\varepsilon_1,\Vert\cdot\Vert_{\infty}\right)
\circ
\mathsf{Cov}\left(\mathcal{G},\varepsilon_2,\Vert\cdot\Vert_{\infty}\right)
$$
is an $(\varepsilon_1+\gamma\varepsilon_2)$-cover for $\mathcal{F}\circ\mathcal{G}$. Thus, the covering number of $\mathcal{F}\circ\mathcal{G}$ can be bounded by
\begin{align}
\log\mathcal{N}\left(\mathcal{F}\circ\mathcal{G},
\varepsilon,\Vert\cdot\Vert_{\infty}\right)
\leq
\inf_{\varepsilon_1,\varepsilon_2\ge0}
\left\{
\log\mathcal{N}\left(\mathcal{F},
\varepsilon_1,\Vert\cdot\Vert_{\infty}\right)
+
\log
\mathcal{N}\left(\mathcal{G},
\varepsilon_2,\Vert\cdot\Vert_{\infty}\right)
\bigg\vert~
\varepsilon_1+\gamma\varepsilon_2\ge\varepsilon
\right\}.
\nonumber
\end{align}
\end{lemma}
\begin{proof}
Let us denote
$
\mathcal{S}\triangleq
\mathsf{Cov}\left(\mathcal{F},\varepsilon_1,\Vert\cdot\Vert_{\infty}\right)
\circ
\mathsf{Cov}\left(\mathcal{G},\varepsilon_2,\Vert\cdot\Vert_{\infty}\right).
$
First, we Obviously have
$\vert\mathcal{S}\vert=
\mathcal{N}(\mathcal{F},\varepsilon_1,\Vert\cdot\Vert_{\infty})\cdot\mathcal{N}(\mathcal{G},\varepsilon_2,\Vert\cdot\Vert_{\infty})$.
On the other hand, for any $f\in\mathcal{F}$ and $g\in\mathcal{G}$, we have
\begin{align}
\inf_{(\widehat{f},\widehat{g})\in\mathcal{S}}~
\sup_{\boldX\in\mathcal{X}}
~
\left\vert
f(g(\boldX))-\widehat{f}(\widehat{g}(\boldX))
\right\vert
\leq&
\inf_{(\widehat{f},\cdot)\in\mathcal{S}}~
\sup_{\boldX\in\mathcal{X}}
~
\left\vert
f(g(\boldX))-\widehat{f}({g}(\boldX))
\right\vert
\nonumber\\
&+
\inf_{(\widehat{f},\widehat{g})\in\mathcal{S}}~
\sup_{\boldX\in\mathcal{X}}
~
\left\vert
\widehat{f}(g(\boldX))-\widehat{f}(\widehat{g}(\boldX))
\right\vert
\nonumber\\
\leq&
\inf_{(\widehat{f},\cdot)\in\mathcal{S}}~
\left\Vert f-\widehat{f}\right\Vert_{\infty}
+
\gamma
\inf_{(\cdot,\widehat{g})\in\mathcal{S}}~
\left\Vert
g-\widehat{g}
\right\Vert_{\infty}.
\end{align}
Due to assumptions and also the property of a covering set, one can deduce
\begin{align}
\inf_{(\widehat{f},\widehat{g})\in\mathcal{S}}~
\sup_{\boldX\in\mathcal{X}}
~
\left\vert
f(g(\boldX))-\widehat{f}(\widehat{g}(\boldX))
\right\vert
\leq
\varepsilon_1+\gamma\varepsilon_2,
\end{align}
which completes the proof.
\end{proof}

\begin{lemma}[Covering Number of $T$-AdaBoost]
\label{lemma:TBoostingCovNum}
Assume a sequence of hypothesis sets denoted by $\mathcal{H}_1,\ldots,\mathcal{H}_T$ for some $T\in\mathbb{N}$, including functions from $\mathcal{X}$ to an arbitrary Euclidean space. Assume 
$\max_t\sup_{h\in\mathcal{H}_t}\Vert h\Vert_{\infty}\leq B_1$ for some fixed and arbitrary value $B_1\ge0$.
Let $T\!-\!\mathsf{AdaBoost}(\mathcal{H}_{1:T})$ be defined as the following function set
\begin{align}
T\!-\!\mathsf{AdaBoost}(\mathcal{H}_{1:T})
\triangleq
\left\{
\sum_{t=1}^{T}\beta_t h_t(\cdot)~\bigg\vert~
h_t\in\mathcal{H}_t~(\forall t\in[T]),~\beta_1,\ldots,\beta_T
\in[-B_2,B_2]
\right\},
\end{align}
for some arbitrary value $B_2\ge0$. Then, the following holds for the covering number of $T\!-\!\mathsf{AdaBoost}(\mathcal{H}_{1:T})$ for any given $\varepsilon>0$:
\begin{align}
\log\mathcal{N}&\left(
T\!-\!\mathsf{AdaBoost}(\mathcal{H}_{1:T}),\varepsilon,\Vert\cdot\Vert_{\infty}
\right)
\\
&\leq
\inf_{\varepsilon_0,\varepsilon_{1},\ldots,\varepsilon_T\ge0}
\left\{
T\log\left(\tfrac{2B_1}{\varepsilon_0}+1\right)
+
\sum_{t=1}^{T}
\log\mathcal{N}\left(
\mathcal{H}_t,\varepsilon_t,\Vert\cdot\Vert_{\infty}
\right)
~\Bigg\vert~
B_1T\varepsilon_0+
B_2\sum_{t=1}^{T}\varepsilon_t\ge\varepsilon
\right\}.
\nonumber
\end{align}
\end{lemma}
\begin{proof}
Fix $\varepsilon>0$ and let $\varepsilon_0,\varepsilon_1,\dots,\varepsilon_T\ge 0$ be arbitrary numbers satisfying
$$
B_1T\varepsilon_0+B_2\sum_{t=1}^T\varepsilon_t\ge\varepsilon.
$$
For each $t\in[T]$ let $\mathsf{Cov}_t$ be an $\varepsilon_t$--net of $\mathcal{H}_t$ with respect to the $\Vert\cdot\Vert_{\infty}$ metric and with minimal cardinality, so
$$
|\mathsf{Cov}_t|=\mathcal{N}\bigl(\mathcal{H}_t,\varepsilon_t,\Vert\cdot\Vert_{\infty}\bigr).
$$
We next construct a finite set of coefficient vectors that will discretize the coefficients $\beta_1,\dots,\beta_T$. For each coordinate $t$ define the one-dimensional grid
$$
G^{(t)}_0=\bigl\{k\varepsilon_0\;\big|\;k\in\mathbb{Z},\ |k\varepsilon_0|\le B_2\bigr\}.
$$
By construction every $\beta_t\in[-B_2,B_2]$ has a nearest neighbor $\tilde\beta_t\in G^{(t)}_0$ with
$
|\beta_t-\tilde\beta_t|\le\varepsilon_0.
$
Let $G_0=\prod_{t=1}^T G^{(t)}_0$ be the Cartesian product grid for the $T$ coefficients. The cardinality of $G^{(t)}_0$ satisfies $|G^{(t)}_0|\le 2B_2/\varepsilon_0+1$, hence
$$
|G_0|\le\prod_{t=1}^T\Bigl(2B_2/\varepsilon_0+1\Bigr)=
\Bigl(2B_2/\varepsilon_0+1\Bigr)^T.
$$
Define the candidate cover $\mathcal{C}$ of $T\!-\!\mathsf{AdaBoost}(\mathcal{H}_{1:T})$ by
$$
\mathcal{C}=\left\{\sum_{t=1}^T \tilde\beta_t \tilde h_t \ \Big|\ \tilde\beta=(\tilde\beta_1,\dots,\tilde\beta_T)\in G_0,\ \tilde h_t\in\mathsf{Cov}_t\right\}.
$$
The cardinality of $\mathcal{C}$ satisfies
$$
|\mathcal{C}|\le |G_0|\cdot\prod_{t=1}^T |\mathsf{Cov}_t|
\le \Bigl(\tfrac{2B_1}{\varepsilon_0}+1\Bigr)^T\cdot\prod_{t=1}^T
\mathcal{N}\bigl(\mathcal{H}_t,\varepsilon_t,\Vert\cdot\Vert_{\infty}\bigr).
$$
It remains to verify that $\mathcal{C}$ is an $\varepsilon$--cover of $T\!-\!\mathsf{AdaBoost}(\mathcal{H}_{1:T})$ in the $\Vert\cdot\Vert_{\infty}$ metric. Let $f\in T\!-\!\mathsf{AdaBoost}(\mathcal{H}_{1:T})$ be arbitrary; write $f=\sum_{t=1}^T\beta_t h_t$ with $h_t\in\mathcal{H}_t$ and $\beta_t\in[-B_2,B_2]$. For each $t$ choose $\tilde\beta_t\in G^{(t)}_0$ with $|\beta_t-\tilde\beta_t|\le\varepsilon_0$ and choose $\tilde h_t\in\mathsf{Cov}_t$ with $\Vert h_t-\tilde h_t\Vert_{\infty}\le\varepsilon_t$. Consider the approximant $\tilde f=\sum_{t=1}^T\tilde\beta_t\tilde h_t\in\mathcal{C}$. Then, for every $\boldX\in\mathcal{X}$,
\begin{align*}
|f(\boldX)-\tilde f(\boldX)|
&=\Bigl|\sum_{t=1}^T\bigl(\beta_t h_t(\boldX)-\tilde\beta_t\tilde h_t(\boldX)\bigr)\Bigr|
\\
&\le\sum_{t=1}^T\Bigl|\beta_t\bigl(h_t(\boldX)-\tilde h_t(\boldX)\bigr)\Bigr|
+\sum_{t=1}^T\Bigl|(\beta_t-\tilde\beta_t)\tilde h_t(\boldX)\Bigr|.
\nonumber
\end{align*}
The first sum is bounded by
$
\sum_{t=1}^T |\beta_t|\cdot\Vert h_t-\tilde h_t\Vert_{\infty}
\le B_2\sum_{t=1}^T\varepsilon_t,
$
since $|\beta_t|\le B_2$. The second sum is bounded by
$$
\sum_{t=1}^T |\beta_t-\tilde\beta_t|\cdot\Vert \tilde h_t\Vert_{\infty}
\le \sum_{t=1}^T \varepsilon_0\cdot B_1 = B_1T\varepsilon_0,
$$
because $\Vert\tilde h_t\Vert_{\infty}\le B_1$ for every $\tilde h_t\in\mathsf{Cov}_t$. Combining the two bounds yields
$$
\Vert f-\tilde f\Vert_{\infty}\le B_1T\varepsilon_0+B_2\sum_{t=1}^T\varepsilon_t.
$$
By the choice of $\varepsilon_0,\varepsilon_1,\dots,\varepsilon_T$ this right-hand side is at least $\varepsilon$, so $\mathcal{C}$ is indeed an $\varepsilon$--cover. Taking logarithms and optimizing over the free parameters $\varepsilon_0,\varepsilon_1,\dots,\varepsilon_T\ge0$ yields
$$
\log\mathcal{N}\bigl(T\!-\!\mathsf{AdaBoost}(\mathcal{H}_{1:T}),\varepsilon,\Vert\cdot\Vert_{\infty}\bigr)
\le
T\log\Bigl(\tfrac{2B_1}{\varepsilon_0}+1\Bigr)
+\sum_{t=1}^T
\log\mathcal{N}\bigl(\mathcal{H}_t,\varepsilon_t,\Vert\cdot\Vert_{\infty}\bigr),
$$
for every choice of $\varepsilon_0,\varepsilon_1,\dots,\varepsilon_T$ such that $B_1T\varepsilon_0+B_2\sum_{t=1}^T\varepsilon_t\ge\varepsilon$. Taking the infimum over such choices gives the claimed bound.
\end{proof}


\subsection{Main Results: Complete Theorems and Corollaries}
\label{sec:app:theory:main}

In this section, we integrate the tools and results developed earlier and provide a complete theoretical generalization analysis for the mixed global–local strategy implemented in Algorithm \ref{alg:ppfe}. We employ a fixed network depth $L$ for all architectures used over the $T$ boosting iterations. For any $T<L$, the boosting mechanism proceeds for $T$ rounds, and at each iteration the number of personalized layers is increased by one, beginning from $t=0$, which corresponds to the fully shared model (FedAvg). At the same time, we allow the layer widths to vary across iterations. In the most general formulation, $N^{(t)}_{l}$ for $l=0,1,\ldots,L$ denotes the width of the network at layer $l$ during iteration $t$. No structural restrictions are imposed on these widths, and the subsequent analysis holds uniformly for arbitrary choices.

Under this formulation, the Empirical Risk Minimization (ERM) step for training the personalized components $\phi^{(t)}_k$ and the shared component $\psi^{(t)}$ at iteration $t$ takes the form
$$
\widehat{\phi}^{(t)}_{1:K},\widehat{\psi}^{(t)}
\gets
\min_{\psi\in\Psi^{(t)}}
\frac{1}{K}\sum_{k=1}^{K}
\min_{\phi_k\in\Phi^{(t)}}
\left\{
\frac{1}{n_k}\sum_{i=1}^{n_k}
\omega^{(t)}_{k,i}
\ell\left(y^{(k)}_i,g_{\phi_k}(f_{\psi}(\boldX^{(k)}_i))\right)
\right\},
$$
where $\widehat{\psi}^{(t)}$ denotes the learned shared parameters at iteration $t$, and $\widehat{\phi}^{(t)}_k$ denotes the learned personalized parameters for client $k$. All these parameter sets depend on the sampled local datasets. Determination of the local boosting weights $\widehat{\beta}^{(t)}_k$ is subsequently carried out on each client following the procedure described in Algorithm \ref{alg:ppfe}. Our objective is to provide a high probability upper bound on the generalization gap,
\begin{align}
\label{eq:app:theory:genGapDef}
\mathsf{Gen}(T,n_{1:K},K)
\triangleq~&
\widehat{\mathcal{L}}_{\mathrm{tot}}\left(
\widehat{\psi}^{(1:T)},\widehat{\phi}^{(1:T)}_{1:K},
\widehat{\beta}^{(1:T)}_{1:K}
\right)
\\
&-
\inf_{\psi^{(1:T)},\phi^{(1:T)}_{1:K},\beta^{(1:T)}_{1:K}}
~\frac{1}{K}\sum_{k=1}^{K}
\mathbb{E}_{P_k}\left[
\ell\!\left(y,
\sum_{t=1}^{T}\beta^{(t)}_k
g_{\phi^{(t)}_k}\!\left(
f_{\psi^{(t)}}\!(\boldX)
\right)
\right)
\right],
\nonumber
\end{align}
where $P_k$ denotes the local true (statistical) distribution of the $k$th client which is assumed to be unknown. Here, the empirical total loss is defined by
$$
\widehat{\mathcal{L}}_{\mathrm{tot}}\left(
{\psi}^{(1:T)},{\phi}^{(1:T)}_{1:K},
{\beta}^{(1:T)}_{1:K}
\right)
\triangleq
\frac{1}{K}\sum_{k=1}^{K}
\left\{
\frac{1}{n_k}\sum_{i=1}^{n_k}
\ell\left(y^{(k)}_i,
\sum_{t=1}^{T}\beta^{(t)}_k
g_{\phi^{(t)}_k}\!\left(f_{\psi^{(t)}}(\boldX^{(k)}_i)\right)\right)
\right\}.
$$
The quantity $\mathsf{Gen}(T,n_{1:K},K)$ quantifies the extent to which Algorithm \ref{alg:ppfe} approaches the optimal solution that would be attainable if the learning process were not federated, that is, if all local datasets were centrally available and the underlying distributions $P_k$ were fully known. Furthermore, $\mathsf{Gen}(T,n_{1:K},K)$ is a random variable because it depends on the particular realizations of the $K$ sampled local datasets. Although AdaBoost performs optimization using weighted empirical losses through the weights $\omega^{(t)}_{k,i}$, ultimately it is the unweighted empirical and statistical risks that characterize generalization performance. The following theorem therefore provides a high probability upper bound on $\mathsf{Gen}(T,n_{1:K},K)$.


\begin{theorem}[Generalization Bound of Algorithm \ref{alg:ppfe}]
\label{thm:main}
Consider a federated learning setup with $K\ge2$ heterogeneous and non-IID clients, each holding $n_k$ local samples, and denote by
$
\bar{n}_{\mathrm{harm}}
\triangleq
K(\frac{1}{n_1}+\cdots+\frac{1}{n_K})^{-1}
$
the harmonic mean of the local sample sizes.  
Assume the loss function $\ell(\cdot)$ is Lipschitz and $B$-bounded for some $B\ge0$, and that all neural networks have uniformly bounded weights according to the statements of Lemmas \ref{lemma:DNNCoveringBound} and \ref{lemma:TBoostingCovNum}. Suppose Algorithm~\ref{alg:ppfe} is run for $T\in\mathbb{N}$ iterations, where in iteration $t\in[T]$ the algorithm uses a fully-connected feed-forward architecture of fixed depth $L\ge T$. Let $N^{(t)}_l$ denote the width of the network used at iteration $t$ in layer $l\in\{0,1,\ldots,L\}$. Then the generalization gap of Algorithm~\ref{alg:ppfe}, denoted by $\mathsf{Gen}(T,n_{1:K},K)$ and defined in \eqref{eq:app:theory:genGapDef}, satisfies
\begin{align}
&\mathsf{Gen}(T,n_{1:K},K)
\leq~
B\sqrt{\frac{\log\frac{2}{\zeta}}{2K\bar{n}_{\mathrm{harm}}}}
~+
\\
&B\sqrt{\frac{\log(K\bar{n}_{\mathrm{harm}})}{\bar{n}_{\mathrm{harm}}}}
\widetilde{\mathcal{O}}
\left[
T+
\sum_{t=1}^{T}
\left(
\frac{1}{K}
\sum_{l=1}^{L-t}
\!\left(
N^{(t)}_lN^{(t)}_{l-1}+N^{(t)}_l
\right)
+
\sum_{l=L-t+1}^{L}\!\left(
N^{(t)}_lN^{(t)}_{l-1}+N^{(t)}_l
\right)
\right)
\right]^{1/2},
\nonumber
\end{align}
with probability at least $1-\zeta$ for any $\zeta\in(0,1)$.
\end{theorem}
\begin{proof}[Proof of Theorem \ref{thm:main}]
To establish a generalization bound for the entire procedure, we derive a uniform convergence guarantee over the space
$$
\Omega\triangleq
\left(\prod_{k=1}^{K}\mathbb{R}^T\right)
\times
\left(\prod_{t=1}^{T}\Psi^{(t)}\right)
\times
\left(\prod_{t=1}^{T}\prod_{k=1}^{K}\Phi^{(t)}\right).
$$
This means that we must form a sup-norm type upper-bound of the generalization gap introduced in \eqref{eq:app:theory:genGapDef}, which can be expressed through the quantity
$$
\mathsf{Gen}(T,n_{1:K},K)
\stackrel{a.s}{\leq}
\sup_{({\beta}^{(t)}_k,\psi^{(t)},\phi^{(t)}_k)\in\Omega}
~
\left\vert
\widehat{\mathcal{L}}_{\mathrm{tot}}
\left({\beta}^{(t)}_k,\psi^{(t)},\phi^{(t)}_k\right)
-
\mathbb{E}\!\left[
\widehat{\mathcal{L}}_{\mathrm{tot}}
\left({\beta}^{(t)}_k,\psi^{(t)},\phi^{(t)}_k\right)
\right]
\right\vert,
$$
and show that this supremum is upper bounded with high probability with respect to the sampling of all private local datasets.

Assume all the neural networks in $\Psi^{(t)}$s and $\Phi^{(t)}$s are covered with a fixed and identical $\ell_{\infty}$-norm error of $\widetilde{\varepsilon}>0$. Then, based on the results of Lemmas \ref{lemma:compositionCovering} and \ref{lemma:TBoostingCovNum}, and due to the Lipschitz assumption on the loss $\ell(\cdot)$, the overall function space will be covered via an $\ell_{\infty}$-norm of
$$
\varepsilon\triangleq
\rho(1+\gamma)(TB_1+TB_2)
\widetilde{\varepsilon},
$$
where $\rho$ and $\gamma$ are the Lipschitz constants of $\ell$ and $f_{\psi}$, respectively, and $B_1$ and $B_2$ are the parameters bounds defined in Lemma \ref{lemma:TBoostingCovNum}. Then, adding the result of Lemma \ref{lemma:DNNCoveringBound} to the previous results, we have
\begin{align}
&\log\mathcal{N}(\Omega,\varepsilon,\Vert\cdot\Vert_{\infty})
\\
&\leq
\sum_{k=1}^{K}T\log\left(1+\tfrac{2B_1}{\widetilde{\varepsilon}}\right)
+
\sum_{t=1}^{T}
\left(
\log\mathcal{N}(\Psi^{(t)},\widetilde{\varepsilon},\Vert\cdot\Vert_{\infty})
+\sum_{k=1}^{K}
\log\mathcal{N}(\Phi^{(t)},\widetilde{\varepsilon},\Vert\cdot\Vert_{\infty})
\right)
\nonumber\\
&\leq
K\!\left(T
+
\sum_{t=1}^{T}\left[
\frac{1}{K}
\sum_{l=1}^{L-t}\left(N^{(t)}_lN^{(t)}_{l-1}+N^{(t)}_l\right)
+
\sum_{l=L-t+1}^{L}\!\left(N^{(t)}_lN^{(t)}_{l-1}+N^{(t)}_l\right)
\right]
\right)
\left(
\log\left(\tfrac{1}{\varepsilon}\right)
+\widetilde{\mathcal{O}}(L)\right)+\widetilde{\mathcal{O}}(1),
\nonumber
\end{align}
where all the dependencies on Lipschitz constants and parameter/weight upper-bounds are logarithmic and thus hidden by $\widetilde{\mathcal{O}}(\cdot)$ in the final term. Note that the last term does not depend on $\varepsilon$.

\begin{lemma}[Pointwise Concentration Bound]
\label{lemma:mcdiarmidBound}
Fix $\psi^{(t)}$s, $\phi^{(t)}_{1:K}$s, and $\beta^{(t)}_k$s for all $t$ and $k$, and assume the loss is bounded as $\lvert \ell(\cdot)\rvert \le B$, where $B\ge 0$ is an arbitrary constant. Then for any $\zeta\in(0,1)$, the empirical weighted loss
$$
\widehat{\mathcal{L}}
=\frac{1}{K}\sum_{k=1}^{K}
\left\{
\frac{1}{n_k}\sum_{i=1}^{n_k}
\ell\left(y^{(k)}_{i}, 
\sum_{t=1}^{T}
\beta^{(t)}_k
g_{\phi^{(t)}_k}
\left(f_{\psi^{(t)}}(\boldX^{(k)}_{i})\right)\right)
\right\}
$$
satisfies
$$
\mathbb{P}\left(
\left|
\widehat{\mathcal{L}}-\mathbb{E}\bigl[\widehat{\mathcal{L}}\bigr]
\right|
\leq
B
\sqrt{\frac{\log\frac{2}{\zeta}}{2K\bar{n}_{\mathrm{harm}}}}
\right)
\ge 1-\zeta,
$$
where $\bar{n}_{\mathrm{harm}}$ is the harmonic mean of local sample sizes $n_{1:K}$, i.e., $K(1/n_1 + \ldots+1/n_K)^{-1}$.
\end{lemma}
\begin{proof}
Fix any parameters $\psi^{(t)}$, $\phi^{(t)}_{1:K}$ and weights $\beta^{(t)}$. For any $k$ and $i$, consider replacing a single sample $(\boldX^{(k)}_i,y^{(k)}_i)$ in client $k$’s dataset. Its contribution to $\widehat{\mathcal{L}}$ is
$$
\frac{1}{K}\cdot \frac{1}{n_k}\cdot
\ell\left(y^{(k)}_i, 
\sum_{t=1}^{T}\beta^{(t)}_k
g_{\phi^{(t)}_k}(f_{\psi^{(t)}}(\boldX^{(k)}_i))\right),
$$
which changes by at most $\frac{B}{K n_k}$. Thus $\widehat{\mathcal{L}}$ has bounded differences with constants
$
c_{k,i} \le \frac{B}{K n_k}.
$
There are $n_1+\cdots+n_K$ independent (but not identically distributed) samples across the $K$ clients. Therefore, McDiarmid’s inequality yields
$$
\mathbb{P}\left(
\left|
\widehat{\mathcal{L}}-\mathbb{E}\bigl[\widehat{\mathcal{L}}\bigr]
\right|
\ge \varepsilon
\right)
\le
2\exp\left(
-2\varepsilon^2\left[
\sum_{k=1}^K\sum_{i=1}^{n_k} c_{k,i}^2
\right]^{-1}
\right).
$$
Substituting the bound on $c_{k,i}$ gives
$$
\sum_{k=1}^K\sum_{i=1}^{n_k} c_{k,i}^2
\le
\sum_{k=1}^K n_k \left(\frac{B}{K n_k}\right)^2
=
\frac{B^2}{K^2}
\sum_{k=1}^K\frac{1}{n_k}
=
\frac{B^2}{K\bar{n}_{\mathrm{harm}}},
$$
where $\bar{n}_{\mathrm{harm}}$ is the \emph{harmonic~average} of $n_1,\ldots,n_K$, i.e., $K(n^{-1}_1+\ldots+n^{-1}_K)^{-1}$. Hence,
$$
\mathbb{P}\left(
\left|
\widehat{\mathcal{L}}-\mathbb{E}\bigl[\widehat{\mathcal{L}}\bigr]
\right|
\ge \varepsilon
\right)
\le
2\exp\left(
\frac{-2K\bar{n}_{\mathrm{harm}}\varepsilon^2}{B^2}
\right).
$$
Setting the RHS equal to $\zeta$ and solving for $\varepsilon$ gives
$$
\varepsilon
\leq
\frac{B}{\sqrt{2K\bar{n}_{\mathrm{harm}}}}
\sqrt{\log\frac{2}{\zeta}},
$$
with probability at least $1-\zeta$ (for any $\zeta\in(0,1)$), completing the proof.
\end{proof}

We now use the result of Lemma \ref{lemma:mcdiarmidBound}, and keep in mind that we have found an $\varepsilon$-cover (in $\ell_{\infty}$-norm) for the total function set $\Omega$. Then, using union bound, we have
\begin{align}
\mathbb{P}\left(
\mathsf{Gen}(T,n_{1:K},K)
\leq
\varepsilon +
B\sqrt{\frac{\log\frac{2}{\zeta}}{2K\bar{n}_{\mathrm{harm}}}}
\right)
\ge 1-\mathcal{N}(\Omega,\varepsilon,\Vert\cdot\Vert_{\infty})\zeta,
\end{align}
for all $\varepsilon>0$ and $\zeta\in(0,1)$. Setting $\zeta\gets\zeta/\mathcal{N}(\Omega,\varepsilon,\Vert\cdot\Vert_{\infty})$, and taking infimum over $\varepsilon>0$, we have
\begin{align}
\mathsf{Gen}(T,n_{1:K},K)
&\leq
\inf_{\varepsilon>0}\left\{
\varepsilon +
B\sqrt{\frac{\log\mathcal{N}(\Omega,\varepsilon,\Vert\cdot\Vert_{\infty}) +\log\frac{2}{\zeta}}{2K\bar{n}_{\mathrm{harm}}}}
\right\}
\nonumber\\
&\leq
\inf_{\varepsilon>0}\left\{
\varepsilon 
+
B\sqrt{\frac{\log\mathcal{N}(\Omega,\varepsilon,\Vert\cdot\Vert_{\infty})}{2K\bar{n}_{\mathrm{harm}}}}
\right\}
+
B\sqrt{\frac{\log\frac{2}{\zeta}}{2K\bar{n}_{\mathrm{harm}}}}
\end{align}
with probability at least $1-\zeta$. It should be noted that for as $\varepsilon$ goes to zero (with increasing $K$ and $\bar{n}$), the terms $\widetilde{\mathcal{O}}(L)$ and $\widetilde{\mathcal{O}}(1)$ become asymptotically negligible, and thus does not have a role in final bound. Therefore, we have
$$
\mathsf{Gen}(T,n_{1:K},K)
\leq
\mathcal{O}\left(
\inf_{\varepsilon>0}\left\{
\varepsilon 
+
B\sqrt{\frac{D\log\tfrac{1}{\varepsilon}}{2K\bar{n}_{\mathrm{harm}}}}
\right\}
\right)
+
B\sqrt{\frac{\log\frac{2}{\zeta}}{2K\bar{n}_{\mathrm{harm}}}},
$$
where
$$
D\triangleq 
KT
+
\sum_{t=1}^{T}\left[
\sum_{l=1}^{L-t}\left(N^{(t)}_lN^{(t)}_{l-1}+N^{(t)}_l\right)
+
K\sum_{l=L-t+1}^{L}\left(N^{(t)}_lN^{(t)}_{l-1}+N^{(t)}_l\right)
\right].
$$

\begin{lemma}
\label{lemma:aepsilon}
Let $a,b>0$, and for $\varepsilon \in (0,1]$ define 
$
F(\varepsilon) \triangleq a \varepsilon + b \sqrt{\log \frac{1}{\varepsilon}}.
$
Then the global minimizer $\varepsilon^\star$ and minimal value $\inf_{\varepsilon\ge0} F(\varepsilon)$ are characterized as follows:
$$
\varepsilon^\star =
\begin{cases}
1, & b > a \sqrt{2/e}, \\
e^{-1/2}, & b = a \sqrt{2/e},\\
e^{-L^\star}, & 0 < b < a \sqrt{2/e}, \;\text{where}~ L^\star > 1/2 \text{ solves } \frac{e^{L^\star}}{\sqrt{L^\star}} = \frac{2a}{b}.
\end{cases}
$$
The minimal value is $F(\varepsilon^\star) = a e^{-L^\star} + b \sqrt{L^\star}$, with the last expression interpreted in the obvious sense for each case. Moreover, in the regime $b \ll a$, the minimizer admits the expansion
$$
\varepsilon^\star = \frac{b}{2a} \left( \log \frac{2a}{b} \right)^{-1/2}+o\left(\frac{b}{a}\right), \qquad 
\inf_{\varepsilon\ge0} F(\varepsilon) = b \sqrt{\log \frac{2a}{b}} + \frac{b}{2} \left( \log \frac{2a}{b} \right)^{-1/2}+o(b).
$$
\end{lemma}
\begin{proof}
We first observe that $F(\varepsilon)$ is continuous on $(0,1]$ and $\lim_{\varepsilon \to 0^+} F(\varepsilon) = +\infty$, $F(1) = a$. Hence $F$ attains a global minimum on $(0,1]$. Introduce the change of variable $L \triangleq \log(1/\varepsilon)$, $\varepsilon = e^{-L}$, so that $L \in [0,\infty)$, and define
$$
G(L) \triangleq F(e^{-L}) = a e^{-L} + b \sqrt{L}.
$$
Then minimizing $F$ over $\varepsilon$ is equivalent to minimizing $G$ over $L$. Compute the derivative of $G$ for $L>0$:
$$
G'(L) = -a e^{-L} + \frac{b}{2\sqrt{L}}\quad,\quad 
G''(L) = a e^{-L} - \frac{b}{4 L^{3/2}}.
$$
Critical points satisfy
\begin{align}
\label{eq:lemma:aepsilon_eq1}
G'(L) = 0 \quad \Longleftrightarrow \quad \frac{e^L}{\sqrt{L}} = \frac{2a}{b}.
\end{align}
Define $\Phi(L) \triangleq e^L / \sqrt{L}$, $L>0$. Then $\Phi'(L) = \frac{e^L}{\sqrt{L}}(1 - 1/(2L))$, so $\Phi$ has a global minimum at $L = 1/2$, with value $\Phi(1/2) = \sqrt{2e}$. Therefore, equation \eqref{eq:lemma:aepsilon_eq1} has: i) No solution if $2a/b < \sqrt{2e}$, equivalently $b > a \sqrt{2/e}$; ii) Exactly one solution $L = 1/2$ if $2a/b = \sqrt{2e}$, equivalently $b = a \sqrt{2/e}$; iii) Two positive solutions $0 < L_1 < 1/2 < L_2$ if $0 < b < a \sqrt{2/e}$. Next, consider the second derivative at a critical point $L$:
$$
G''(L) = \frac{b}{4 L^{3/2}} (2L - 1).
$$
Hence $L_2 > 1/2$ corresponds to a local minimum, while $L_1 < 1/2$ corresponds to a local maximum. Since $G(L) \to +\infty$ as $L \to \infty$ and $G$ is continuous, the global minimum occurs at $L^\star = L_2$. The borderline case $b = a \sqrt{2/e}$ corresponds to the unique critical point $L^\star = 1/2$. If $b > a \sqrt{2/e}$, no critical point exists and the minimizer is at the endpoint $L = 0$ ($\varepsilon = 1$). Thus, the global minimizer and minimal value are precisely
$$
\varepsilon^\star =
\begin{cases}
1, & b > a \sqrt{2/e}, \\
e^{-1/2}, & b = a \sqrt{2/e}, \\
e^{-L^\star}, & 0 < b < a \sqrt{2/e},\; L^\star > 1/2 \text{ solves \eqref{eq:lemma:aepsilon_eq1}},
\end{cases}
\qquad
F(\varepsilon^\star) = a e^{-L^\star} + b \sqrt{L^\star}.
$$

\noindent
\textbf{Asymptotic expansion for $b \ll a$.} Let $A \triangleq \log(2a/b) \to \infty$ as $b \to 0$. Equation \eqref{eq:lemma:aepsilon_eq1} can be rewritten as
$$
L^\star = \log \frac{2a}{b} + \frac12 \log L^\star = A + \frac12 \log L^\star.
$$
Assuming $L^\star = A + \delta$, substituting yields $\delta = \frac12 \log(A + \delta)$. Since $A \to \infty$, the solution admits the expansion
$
L^\star = A + \frac12 \log A + o(1).
$
Therefore
$$
\varepsilon^\star = e^{-L^\star} = \frac{b}{2a} (\log (2a/b))^{-1/2}+o(b/a).
$$
The minimal value is
$$
F(\varepsilon^\star) = a e^{-L^\star} + b \sqrt{L^\star} 
= b \sqrt{\log \frac{2a}{b}} + \frac{b}{2} \left(\log \frac{2a}{b}\right)^{-1/2}+o(b).
$$
This concludes the proof.
\end{proof}
Using Lemma \ref{lemma:aepsilon}, we have
\begin{align}
\mathbb{P}\left(
\mathsf{Gen}(T,n_{1:K},K)
\leq
B\sqrt{D\frac{\log(K\bar{n}_{\mathrm{harm}})}{2K\bar{n}_{\mathrm{harm}}}}
+
B\sqrt{\frac{\log\frac{2}{\zeta}}{2K\bar{n}_{\mathrm{harm}}}}
+
c\cdot(K\bar{n}_{\mathrm{harm}})^{-1/2}
\right)\ge 1-\zeta,
\nonumber
\end{align}
for some universal constant $c>0$ and all $\zeta\in(0,1)$. This completes the proof.
\end{proof}


\begin{corollary}[Special Neural Network Architectures]
\label{corl:specialSetting}
Consider the setting of Theorem~\ref{thm:main} for neural networks of depth $L>T$. Assume all clients have equal local sample sizes $n$, and that the width of the common layers of the DNNs is fixed to a base width $W_b \ge 1$ for all iterations $t \in [T]$. In contrast, the widths of the personalized layers are assumed to gradually decrease across iterations according to $W_b t^{-\alpha}$ for some user-defined $\alpha \ge 0$. Then the generalization gap of PPFE satisfies
\begin{align}
\mathsf{Gen}(T,n_{1:K},K)
~\le~
(W_b+1)
\widetilde{\mathcal{O}}\!\left(
\left(\frac{LT}{Kn}\right)^{1/2}
+
\left(\frac{\log(1/\zeta)}{Kn}\right)^{1/2}
+
\frac{T^{1-\alpha}}{n^{1/2}}
\right)
+
\widetilde{\mathcal{O}}\left(
\frac{T^{1/2}}{n^{1/2}}
\right),
\nonumber
\end{align}
with probability at least $1-\zeta$, for any $\zeta \in (0,1)$. The hidden constants depend at worst polylogarithmically on the Lipschitz constants and weight bounds specified in Theorem~\ref{thm:main}.
\end{corollary}
\begin{proof}[Proof of Corollary \ref{corl:specialSetting}]
Proof is straightforward. Take the generalization bound of Theorem \ref{thm:main}. First, we have
$$
\bar{n}_{\mathrm{harm}}=K(1/n+\ldots+1/n)^{-1}=n.
$$
Next, due to the assumptions on neural architecture settings, we have
\begin{align}
\sum_{t=1}^{T}
\sum_{l=1}^{L-t}
\!\left(
N^{(t)}_lN^{(t)}_{l-1}+N^{(t)}_l
\right)
~&=~
\sum_{t=1}^{T}
(L-t)(W^2_b+W_b)
~\leq~
LT(W_b+1)^2,
\nonumber\\
\mathrm{and}\hspace*{1.1cm}
K\!\sum_{t=1}^{T}\sum_{l=L-t+1}^{L}\!\left(
N^{(t)}_lN^{(t)}_{l-1}+N^{(t)}_l
\right)
~&=~
K\sum_{t=1}^{T}
t\left(W^2_bt^{-2\alpha}+W_bt^{-\alpha}\right)
\nonumber\\
&\leq~
KW_b\int_{0}^{T}
\left(W_bt^{1-2\alpha}
+
t^{1-\alpha}
\right)\mathrm{d}t
\nonumber\\
&=~
KW_b\left[
W_b
\left(\frac{T^{2(1-\alpha)}}{2(1-\alpha)}\right)
+
\frac{T^{2-\alpha}}{2-\alpha}
\right],
\end{align}
where, since we usually have $W_bT^{-\alpha}\gg 1$ (i.e., the personalized layers of the last iteration still have many neurons), we only consider the $W_b({T^{2(1-\alpha)}}/{(2-2\alpha)})$ term in the second inequality. Therefore, the bound in Theorem \ref{thm:main} becomes:
\begin{align}
\mathsf{Gen}(T,n_{1:K},K)
&\leq~
B\sqrt{\frac{\log\frac{2}{\zeta}}{2Kn}}
+
B\sqrt{\frac{\log(Kn)}{Kn}}
\widetilde{\mathcal{O}}
\left[
KT+
LT(W_b+1)^2
+
2KW^2_b\left(\frac{T^{2(1-\alpha)}}{2(1-\alpha)}\right)
\right]^{1/2},
\nonumber\\
&\leq~
(W_b+1)
\widetilde{\mathcal{O}}\!\left(
\left(\frac{LT}{Kn}\right)^{1/2}
+
\left(\frac{\log(1/\zeta)}{Kn}\right)^{1/2}
+
\frac{T^{1-\alpha}}{n^{1/2}}
\right)
+
\widetilde{\mathcal{O}}\left(
\frac{T^{1/2}}{n^{1/2}}
\right),
\end{align}
and the proof is complete.
\end{proof}

\end{document}